\documentclass{article}

\PassOptionsToPackage{numbers, compress}{natbib}
\usepackage[final]{neurips_2021}

\usepackage{txie}

\def\shownotes{0}  %
\ifnum\shownotes=1
\newcommand{\authnote}[2]{\noindent$^{\text{\fontfamily{cmtt}\em #1:}}\langle${\sf\small #2}$\rangle$}
\else
\newcommand{\authnote}[2]{}
\fi

\newcommand{\approach}{Bellman-consistent\xspace }
\newcommand{\erron}{\mathrm{err}_{\mathrm{on}}}
\newcommand{\erroff}{\mathrm{err}_{\mathrm{off}}}

\usepackage[utf8]{inputenc} %
\usepackage[T1]{fontenc}    %
\usepackage{hyperref}       %
\usepackage{url}            %
\usepackage{booktabs}       %
\usepackage{amsfonts}       %
\usepackage{nicefrac}       %
\usepackage{microtype}      %
\usepackage{xcolor}         %

\title{\approach Pessimism for \\ Offline Reinforcement Learning}

\author{%
  Tengyang Xie \\
  UIUC\\
  \href{mailto:tx10@illinois.edu}{\texttt{tx10@illinois.edu}} \\
  \And
  Ching-An Cheng \\
  Microsoft Research\\
  \href{mailto:chinganc@microsoft.com}{\texttt{chinganc@microsoft.com}} \\
  \And
  Nan Jiang \\
  UIUC\\
  \href{mailto:nanjiang@illinois.edu}{\texttt{nanjiang@illinois.edu}} \\
  \And
  Paul Mineiro \\
  Microsoft Research\\
  \href{mailto:pmineiro@microsoft.com}{\texttt{pmineiro@microsoft.com}}
  \And
  Alekh Agarwal \\
  Google Research\\
 \href{mailto:alekhagarwal@google.com}{\texttt{alekhagarwal@google.com}} \\
}

\begin{document}

\maketitle

\begin{abstract}
The use of pessimism, when reasoning about %
datasets lacking exhaustive exploration, has recently gained prominence in offline reinforcement learning. %
Despite the robustness it adds to the algorithm, %
overly pessimistic reasoning can be equally damaging in precluding the discovery of good policies, which is an issue for the popular bonus-based pessimism. In this paper, we introduce the notion of \approach pessimism for general function approximation: %
instead of calculating a point-wise lower bound for the value function, we implement pessimism at the initial state over the set of functions consistent with %
the Bellman equations. %
Our theoretical guarantees only require Bellman closedness as standard in the exploratory setting, in which case bonus-based pessimism fails to provide guarantees. Even in the special case of linear function approximation where stronger expressivity assumptions hold, 
our result improves upon a recent bonus-based approach by $\mathcal{O}(d)$ in its sample complexity when the action space is finite and small. 
Remarkably, our algorithms automatically adapt to the best bias-variance tradeoff in the hindsight, whereas most prior approaches require tuning extra hyperparameters a priori.
\end{abstract}

\section{Introduction}
\label{sec:intro}
Using past experiences to learn improved behavior for future interactions is a critical capability for a Reinforcement Learning (RL) agent. 
However, robustly extrapolating knowledge from a historical dataset for sequential decision making is highly challenging, particularly in settings where function approximation is employed to generalize across related observations.
In this paper, we provide a systematic treatment of such scenarios with \emph{general} function approximation, and devise algorithms that can provably leverage an arbitrary historical dataset to discover the policy that obtains the largest guaranteed rewards, amongst all possible scenarios consistent with the dataset.

The problem of learning a good policy from historical datasets, typically called batch or offline RL, has a long history~\citep[see e.g.,][and references therein]{precup2000eligibility,antos2008learning,levine2020offline}. Many prior works~\citep[e.g.,][]{precup2000eligibility,antos2008learning,chen2019information} make the so-called \emph{coverage assumptions} on the dataset, requiring the dataset  to contain any possible state, action pair or trajectory with a lower bounded probability. These assumptions are evidently prohibitive in practice, particularly for problems with large state and/or action spaces. Furthermore, the methods developed under these assumptions routinely display unstable behaviors such as lack of convergence or error amplification, when coverage assumptions are violated~\citep{wang2020statistical,wang2021instabilities}. 

Driven by these instabilities, a growing body of recent literature has pursued a so-called \emph{best effort} style of guarantee instead. The key idea is to replace the stringent assumptions on the dataset  with a dataset-dependent performance bound, which gracefully degrades from guaranteeing a near-optimal policy under standard coverage assumptions to offering no improvement over the data collection policy in the most degenerate case. Algorithmically, these works all leverage the principle of \emph{pessimistic extrapolation} from offline data and aim to maximize the rewards the trained agent would obtain in the worst possible MDP that is consistent with the observed dataset.
These methods have been shown to be typically more robust to the violation of coverage assumptions in practice, and their theoretical guarantees often provide non-trivial conclusions in settings where the previous results did not apply.

Even though many such best-effort methods have now been developed, %
very few works provide a comprehensive theory for using generic function approximation, unlike the setting where the dataset satisfies the coverage assumptions~\citep{antos2008learning,munos2003error,szepesvari2005finite,munos2008finite,farahmand2010error,chen2019information,xie2020q}. For example,~\citep{kidambi2020morel} provides a partial theory  under the assumption of an uncertainty quantification oracle, which however is highly nontrivial to obtain for general function approximation. \citep{fujimoto2019off, kumar2020conservative} develop sound theoretical arguments in the tabular setting, which were only heuristically extended to the function approximation setting. The works that explicitly consider function approximation in their design either use an ad-hoc truncation of Bellman backups~\citep{liu2020provably} or strongly rely on particular parameterizations such as linear function approximation~\citep{jin2020pessimism}. In particular, \citep{liu2020provably} additionally requires the ability to approximate stationary distribution of the behavior policy, which is a challenging density estimation problem for complex state spaces and cannot be provably performed in the standard linear MDP setting (see Section~\ref{sec:constrdpi_linear}).

Our paper takes an important step in this direction. We provide a systematic way to encode pessimism compatible with an arbitrary function approximation class and MDP and give strong theoretical guarantees \emph{without} requiring any coverage assumptions on the dataset. 
Our first contribution is an information theoretic algorithm that returns a policy with a small regret to any comparator policy, for which coverage assumptions (approximately) hold with respect to the data collection policy. This regret bound is identical to what can be typically obtained when the coverage assumptions hold \emph{for all policies}~\citep{antos2008learning,chen2019information}. But our algorithm requires neither the coverage assumptions, nor additional assumptions such as reliable density estimation for the data generating distribution used by existing best-effort approaches~\citep{liu2020provably}. 
We furthermore instantiate these results in the special case of linear parameterization; under the linear MDP assumption, our sample complexity bound leads to a factor of $\Ocal(d)$ improvement for a $d$-dimensional linear MDP, compared with the best known result translated to our discounted setting~\citep{jin2020pessimism}, when the action set is small in size. 
In addition to the information theoretic algorithm, 
we also develop a computationally practical version of our algorithm using a Lagrangian relaxation combined with recent advances in soft policy iteration~\citep{even2009online,geist2019theory,agarwal2019theory}. 
We show that this algorithm can be executed efficiently by querying a (regularized) loss minimization oracle over the value function class, although it has slightly worse theoretical guarantees than the information theoretic version. Both our algorithms display an adaptive property in selecting the best possible form of a bias-variance decomposition, where most prior approaches had to commit to a particular point through their choice of hyperparameters (see the discussion following Theorem~\ref{thm:infothebd2}).  

\section{Preliminaries}
\paragraph{Markov Decision Processes} We consider dynamical systems modeled as Markov Decision Processes (MDPs). An MDP is specified by $(\Scal, \Acal, P, R, \gamma, s_0)$, where $\Scal$ is the state space, $\Acal$ is the action space, $P: \Scal\times\Acal \to \Delta(\Scal)$ is the transition function with $\Delta(\cdot)$ being the probability simplex, $R:\Scal\times\Acal\to[0, \Rmax]$ is the reward function, %
$\gamma \in [0, 1)$ is the discount factor, and $s_0$ is a deterministic initial state, which is without loss of generality. We assume the state and the action spaces are finite but can be arbitrarily large. %
A (stochastic) policy $\pi: \Scal\to\Delta(\Acal)$ specifies a decision-making strategy, and induces a random trajectory $s_0, a_0, r_0, s_1, a_1, r_1, \ldots$, where $a_t \sim \pi(\cdot|s_t)$, $r_t = R(s_t, a_t)$, $s_{t+1} \sim P(\cdot|s_t,a_t)$, $\forall t\ge 0$. We denote the expected discounted return of a policy $\pi$ as %
$J(\pi):= \E[\sum_{t=0}^\infty \gamma^t r_t | \pi]$, and the learning goal is to find the maximizer of this value: $\pi^\star \coloneqq \argmax_{\pi} J(\pi)$. A related concept is the \textit{policy-specific} $Q$-function, $Q^\pi: \Scal\times\Acal\to\RR$. $Q^\pi(s,a)$ is the discounted return when the trajectory starts with $(s,a)$ and all remaining actions are taken according to $\pi$. $Q^\pi$ is the unique fixed point of the (policy-specific) Bellman operator $\Tcal^\pi: \RR^{\Scal\times\Acal} \to \RR^{\Scal\times\Acal}$, defined as: 
\begin{equation*}
    \forall f, \quad(\Tcal^\pi f)(s,a) = R(s,a) + \gamma \E_{s'\sim P(\cdot|s,a)}[f(s',\pi)],
\end{equation*} 
where $f(s',\pi)$ is a shorthand for $\E_{a' \sim \pi(\cdot|s')}[f(s', a')]$. 

Another important concept is the notion of discounted state-action occupancy, $d_\pi \in \Delta(\Scal\times\Acal)$, 
defined as $d_\pi(s,a):= (1 - \gamma)\E[\sum_{t=0}^\infty \gamma^t \1[s_t = s, a_t=a]|\pi]$, which characterizes the states and actions visited by a policy $\pi$.

\paragraph{Offline RL} In the offline setting, the learner only has access to a pre-collected dataset and cannot directly interact with the environment. We assume the standard i.i.d.~data generation protocol in our theoretical derivations, that the offline dataset $\Dcal$ consists of $n$ i.i.d.~$(s,a,r,s')$ tuples generated as $(s,a)\sim \mu, r= R(s,a), s'\sim P(\cdot|s,a)$ for some  \textit{data distribution} $\mu$. %
We will also use $\E_{\mu}[\cdot]$ for taking expectation with respect to $\mu$.
We will frequently use the data-weighted 2-norm (squared) $\|f\|_{2, \mu}^2:= \E_\mu[f^2]$, and the definition extends when we replace $\mu$ with any other state-action distribution $\nu$. The empirical approximation of $\|f\|_{2, \mu}^2$ is $\|f\|_{2, \Dcal}^2 := \frac{1}{n}\sum_{(s,a,r,s') \in \Dcal} f(s,a)^2$.

\paragraph{Function Approximation} Function approximation is crucial to generalizing over large and complex state and action spaces. In this work, we search for a good policy in a policy class $\Pi\subset (\Scal\to\Delta(\Acal))$ with the help of a value-function class $\Fcal \subset (\Scal\times\Acal\to[0, \Vmax])$ to model $Q^\pi$, where $\Vmax = \Rmax/(1-\gamma)$. Such a combination is commonly found in approximate policy iteration and actor-critic algorithms~\citep[e.g.,][]{bertsekas1996neuro,konda2000actor}.
For most part of the paper we do not make any structural assumptions on $\Pi$ and $\Fcal$, making our approach and guarantees applicable to generic function approximators. For simplicity we will assume that these function classes are finite but exponentially large, and use log-cardinality to measure their statistical complexities in the generic results (Section~\ref{sec:constrdpi} and Section~\ref{sec:regulpi}). These guarantees easily extend to continuous function classes where log-cardinalities are replaced by the appropriate notions of covering numbers, which we demonstrate when we instantiate our results in the linear function approximation setting and work with continuous linear classes (Section~\ref{sec:constrdpi_linear}).

We now recall two standard expressivity assumptions on $\Fcal$~\cite[e.g.,][]{antos2008learning}. To our knowledge, no existing works on offline RL with insufficient data coverage have provided guarantees under these standard assumptions for general function approximation, and they often require stronger or tweaked assumptions (see Section~\ref{sec:intro}). 
\begin{assumption}[Realizability]
\label{asm:relz2}
For any $\pi \in \Pi$, we have
\begin{align}
\inf_{f \in \Fcal}\sup_{\text{admissible }\nu} \left\|f - \Tcal^{\pi} f\right\|_{2,\nu}^2 \leq \varepsilon_\Fcal,
\end{align}
where an admissible distribution $\nu$ means that $\nu \in \{d_{\pi'}: \pi' \in \Pi\}\cup \{\mu\}$.
\end{assumption}
Assumption~\ref{asm:relz2} requires that for every $\pi\in\Pi$, there exists $f \in \Fcal$ that well-approximates $Q^\pi$. This assumption is often called \textit{realizability}.\footnote{In the exploratory setting, realizability is usually stated in the form of $\inf_{f \in \Fcal} \left\|f - \Tcal^\pi f\right\|_{2,\mu}^2 \leq \varepsilon_\Fcal$ for any $\pi \in \Pi$. However, the exploratory setting also usually has data coverage assumptions in the form of $\sup_{\nu} \|\nicefrac{\nu}{\mu}\|_\infty \leq C$. Combining them together implies Assumption~\ref{asm:relz2}.} Technically this is asserted by requiring $f$ to have small Bellman error w.r.t.~$\Tcal^\pi$ under all possible \textit{admissible} distributions. As a sufficient condition, we have $\varepsilon_\Fcal =0$ if $Q^\pi \in \Fcal, \forall \pi\in\Pi$. 
\begin{assumption}[Completeness]
\label{asm:comp_restat}
For any $\pi \in \Pi$, we have
\begin{align}
\sup_{f \in \Fcal} \inf_{f' \in \Fcal} \left\|f' - \Tcal^{\pi} f \right\|_{2,\mu}^2 \leq \varepsilon_{\Fcal,\Fcal}.
\end{align}
\end{assumption} Assumption~\ref{asm:comp_restat} asserts that $\Fcal$ is approximately closed under $\Tcal^\pi$.\footnote{Sometimes completeness implies realizability, so the latter does not need to be assumed separately~\citep{chen2019information}. However, this often relies on $\mu$ being exploratory, which is not the case here. } %
Such an assumption is widely used in RL theory and can be only avoided in some rare cases~\citep{xie2021batch}, and the hardness of learning with realizability alone has been established in various settings (e.g.,~\citep{weisz2021exponential,zanette2021exponential}). We also emphasize that we only measure the violation of completeness under $\mu$ and do not need to reason about all admissible distributions. 

\paragraph{Distribution shift} 
A unique challenge in RL is that the learned policy may induce a state (and action) distribution that is different from the data distribution $\mu$, and the issue is particularly salient when we do not impose any coverage assumption on $\mu$. Therefore, it is important to carefully characterize the distribution shift, which we measure using the following definition, which generalizes prior definitions specific to linear function approximation~\citep{agarwal2019theory,duan2020minimax}:
\begin{definition}
\label{def:concenbddist}
    We define $\Cscr(\nu;\mu,\Fcal,\pi)$ as follows to measure the distribution shift from an arbitrary distribution $\nu$ to the data distribution $\mu$, w.r.t.~$\Fcal$ and $\pi$, %
\begin{align}
\label{eq:concenbddist}
\Cscr(\nu;\mu,\Fcal, \pi) \coloneqq \max_{f \in \Fcal}\frac{\|f - \Tcal^\pi f\|_{2,\nu}^2}{\|f - \Tcal^\pi f\|_{2,\mu}^2}.
\end{align}
\end{definition}
Intuitively, $\Cscr(\nu;\mu,\Fcal, \pi)$ measures how well Bellman errors under $\pi$ transfer between the distributions $\nu$ and $\mu$. For instance, a small value of $\Cscr(d_\pi; \mu, \Fcal,\pi)$ enables accurate policy evaluation for $\pi$ using data collected under $\mu$. More generally, we observe that  
\begin{equation}
    \Cscr(\nu;\mu,\Fcal, \pi) \leq \|\nu/\mu\|_\infty := \sup_{s,a} \frac{\nu(s,a)}{\mu(s,a)},\quad\mbox{for any $\pi$, $\Fcal$.}
\end{equation}
and the RHS is a classical notion of bounded distribution ratio for error transfer (e.g.,~\citep{munos2008finite,chen2019information,xie2020q}). Moreover, our measure can be tighter than $\|\nu/\mu\|_\infty$: Even two distributions $\nu$ and $\mu$ that are sufficiently disparate might admit a reasonable transfer, so long as this difference is not detected by $\pi$ and $\Fcal$. To this end, our definition %
better captures the crucial role of function approximation in generalizing across different states. As an example, in the special case of linear MDPs, full coverage under our definition (i.e., boundedness of $\Cscr$ \textit{for all admissible $\nu$}) can be implied from the standard coverage assumption for linear MDPs that considers the spectrum of the feature covariance matrix under $\mu$; see Section~\ref{sec:constrdpi_linear} for more details.

\section{Information-Theoretic Results with Bellman-consistent Pessimism}
\label{sec:constrdpi}

In this section, we provide our first theoretical result which is information-theoretic, in that it uses a computationally inefficient algorithm. The approach uses the offline dataset to first compute a lower bound on the value of each policy $\pi \in \Pi$, and then returns the policy with the highest pessimistic value estimate. While this high-level template is at the heart of many recent approaches~\citep[e.g.,][]{fujimoto2019off,kumar2019stabilizing,liu2020provably,kidambi2020morel,yu2020mopo,kumar2020conservative}, our main novelty is in the design and analysis of \emph{Bellman-consistent pessimism} for general function approximation.

For a policy $\pi$, we first form a \emph{version space} of all the functions $f\in\Fcal$ which have a small Bellman error under the evaluation operator $\Tcal^\pi$. We then return the predicted value of $\pi$ in the initial state $s_0$ by the functions in this version space. The use of pessimism at the initial state, while maintaining Bellman consistency (by virtue of having a small Bellman error) limits over pessimism, which is harder to preclude in the pointwise pessimistic penalties used in some other works~\citep{jin2020pessimism}. 
More formally, given a dataset $\Dcal$, let us define %
$$\Lcal(f',f, \pi;\Dcal) \coloneqq \frac{1}{n} \sum_{(s,a,r,s') \in \Dcal} \left(f'(s,a) - r - \gamma f(s',\pi) \right)^2,$$
and an empirical estimate of the Bellman error $\Ecal(f,\pi;\Dcal)$ is
\begin{equation}
\label{eq:defmsbope}
\begin{aligned}
&~ \Ecal(f,\pi;\Dcal) \coloneqq \Lcal(f,f,\pi;\Dcal) - \min_{f' \in \Fcal}\Lcal(f',f,\pi;\Dcal).
\end{aligned}
\end{equation}

\noindent\paragraph{Our algorithm.} With this notation, our information-theoretic approach finds a policy by optimizing:%
\begin{align}
\label{eq:infotheosol}
\pihat = \argmax_{\pi \in \Pi} \min_{f \in \Fcal_{\pi,\varepsilon}} f(s_0,\pi), \text{\quad where } \Fcal_{\pi, \varepsilon} = \{f \in \Fcal: \Ecal(f,\pi;\Dcal) \leq \varepsilon\},
\end{align}
In the formulation above, $\Fcal_{\pi, \varepsilon}$ is the version space of policy $\pi$. To better understand the intuition behind the estimator in Equation~\ref{eq:infotheosol}, let us define
\begin{equation}
f_{\pi,\min} \coloneqq \argmin_{f \in \Fcal_{\pi,\varepsilon}}f(s_0,\pi), ~~ f_{\pi,\max} \coloneqq \argmax_{f \in \Fcal_{\pi,\varepsilon}}f(s_0,\pi), ~~\mbox{and}~~\Delta f_{\pi}(s,a) \coloneqq f_{\pi,\max}(s,a) - f_{\pi,\min}(s,a).
    \label{eq:deltaf}
\end{equation}

Intuitively, if the parameter $\varepsilon$ is defined to ensure that $Q^\pi$ (or its best approximation in $\Fcal$) is in $\Fcal_{\pi, \varepsilon}$, we easily see that $\Delta f_\pi(s_0,\pi)$ is an upper bound on the error in our estimate of $J(\pi)$ for any $\pi\in\Pi$. In fact, an easy argument in our analysis shows that if $Q^\pi \in \Fcal_{\pi, \varepsilon}$ for all $\pi\in\Pi$, then $\Delta f(s_0,\pi)$ is an upper bound on the regret $J(\pi) - J(\pihat)$ of our estimator relative to any $\pi$ we wish to compete with. 

\noindent\paragraph{Theoretical analysis.} To leverage this observation, we first define the a critical threshold $\varepsilon_r$ which ensures that (the best approximation of) $Q^\pi$ is indeed contained in our version spaces for all $\pi$:

\begin{align}
\label{eq:def_varepsilon_r}
\varepsilon_r \coloneqq \frac{139 \Vmax^2 \log \frac{|\Fcal||\Pi|}{\delta}}{n} + 39 \varepsilon_{\Fcal}.
\end{align}
With this definition, we now give a more refined bound on the regret of our algorithm~\eqref{eq:infotheosol} by further splitting the error estimate $\Delta f_\pi(s_0,\pi)$ which is random owing to its dependence on the version space, and analyze it through a novel decomposition into on-support and off-support components. While we bound the on-support error using standard techniques, the off-support error is akin to a bias term which captures the interplay between the data collection distribution and function approximation in the quality of the final solution.
Also note that our choice of $\varepsilon_r$ requires the knowledge of $\varepsilon_\Fcal$, which is a common characteristic of version-space-based algorithms~\citep[e.g.,][]{jiang2017contextual}. %
The challenge of unknown $\varepsilon_{\Fcal}$ can be possibly addressed using model-selection techniques in practice and we leave further investigation to future work.

\begin{theorem}
\label{thm:infothebd2}
Let $\varepsilon = \varepsilon_r$ where is $\varepsilon_r$ defined in \Eqref{eq:def_varepsilon_r} and $\pihat$ be obtained by \Eqref{eq:infotheosol}. Then, for any policy $\pi \in \Pi$ and any constant $C_2 \ge 1$, with probability at least $1-\delta$, 
\begin{align}
J(\pi) - J(\pihat) \leq &~ \underbrace{\Ocal\left( \frac{\Vmax \sqrt{C_2}}{1 - \gamma} \sqrt{\frac{\log \frac{|\Fcal||\Pi|}{\delta}}{n}} + \frac{\sqrt{C_2 (\varepsilon_{\Fcal,\Fcal} + \varepsilon_\Fcal)}}{1 - \gamma} \right)}_{\text{$\erron(\pi)$: on-support error}} \\
&~ + \underbrace{\frac{1}{1 - \gamma}\cdot\,\min_{\nu: \Cscr(\nu;\mu,\Fcal,\pi) \leq C_2}   \sum_{(s,a) \in \Scal\times\Acal} (d_{\pi}\setminus\nu)(s,a) \left[\Delta f_{\pi}(s,a) - \gamma(\Pcal^\pi \Delta f_{\pi})(s,a)\right] }_{\text{$\erroff(\pi)$: off-support error}},
\end{align}
where $\Cscr(\nu;\mu,\Fcal,\pi)$ is defined in Definition~\ref{def:concenbddist}, %
$(d_{\pi}\setminus\nu)(s,a) \coloneqq \max(d_\pi(s,a) - \nu(s,a),0)$ and $(\Pcal^\pi f)(s,a) = \E_{s'\sim P(\cdot|s,a)}[f(s',\pi)]$ for any $f$.
\end{theorem}

\begin{wrapfigure}{r}{0.4\textwidth}
\vspace{-5mm}
\centering
\includegraphics[width=\linewidth]{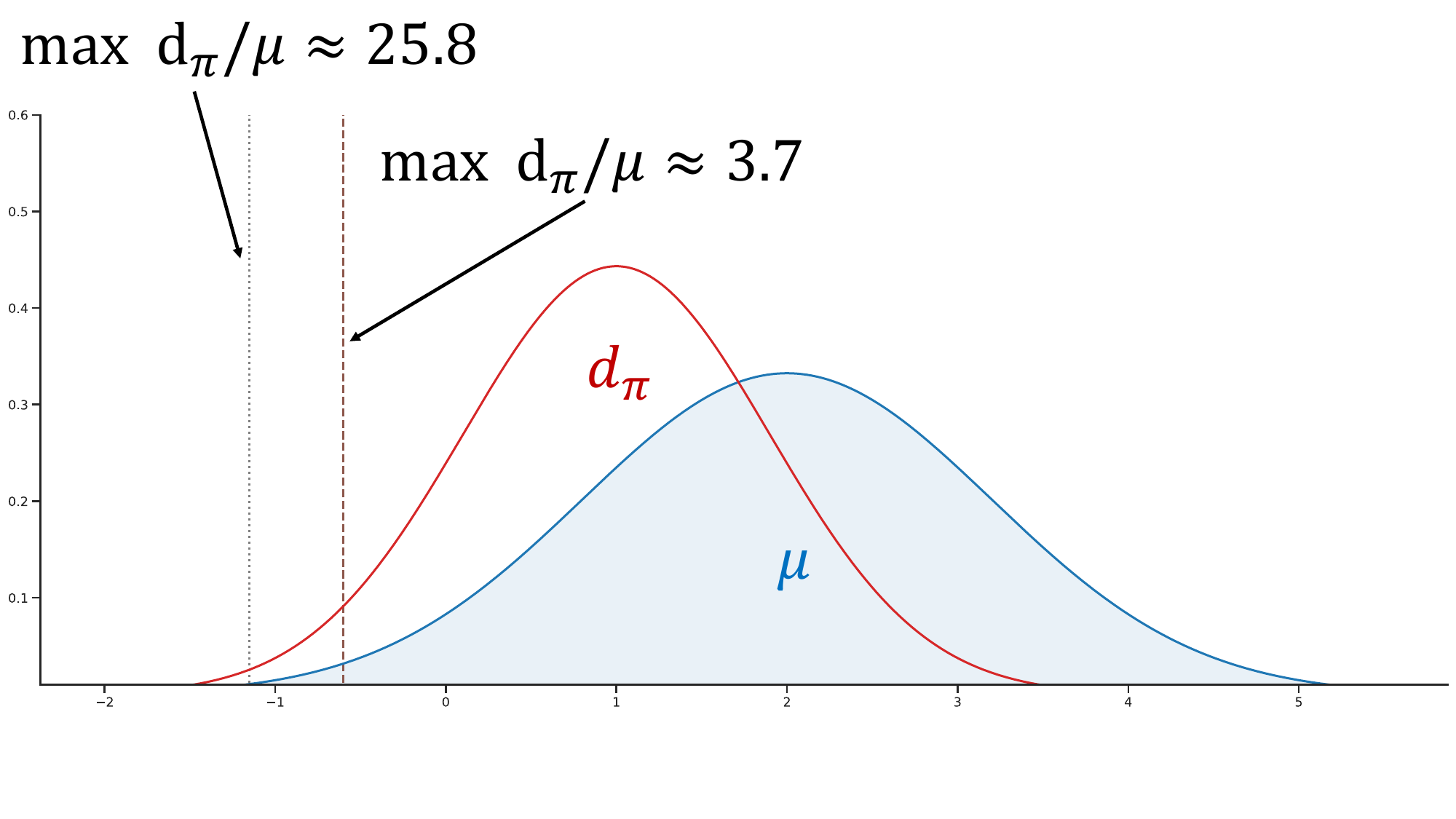}
\vspace{-7mm}
\caption{An example illustrating different on-support and off-support splittings (denoted by two different vertical lines). Different splitting has different $C_2$ values, and further yields different bias-variance trade-offs.}
\label{fig:split}
\end{wrapfigure}

\paragraph{Bias-variance decomposition.} Note that decomposition of our error bound into on-support and off-support parts effectively achieves a bias-variance tradeoff. A small value of the concentrability threshold $C_2$ requires the choice of the distribution $\nu$ closer to $\mu$, which results in better estimation error guarantee (which is $\Ocal(\sqrt{\nicefrac{C_2}{n}})$) %
when we transfer from $\mu$ to $\nu$, but potentially pays a high bias due to the mismatch between $d_\pi$ and $\nu$. A larger threshold permits more flexibility in choosing $\nu$ similar to $d_\pi$ for a smaller bias, but results in a larger variance and estimation error. %
Rather than commit to a particular tradeoff, our estimator automatically adapts to the best possible splitting (Figure~\ref{fig:split} illustrates this concept) by allowing us to choose the best threshold $C_2$. The on-support part matches the $n$ rate (fast rate error bound) of API or AVI analysis (e.g.,~\citep{pires2012statistical,lazaric2012finite,chen2019information}). The dependency on horizon is only linear and matches the best previous result with concentrability assumption~\citep{xie2020q}. For the off-support part, it depends on the off-support mass $d_\pi\setminus \nu$ and the ``quality'' of the off-support estimation: if all value functions in the version space are close to each other in the off-support region for policy $\pi$, the gap between $J(\pi)$ and $J(\pihat)$ can still be small even with a large off-support mass. The following corollary formally states this property. 
\begin{corollary}[``Double Robustness''] \label{cor:DR}
\label{cor:info_theo_dr}
Under conditions of Theorem~\ref{thm:infothebd2}, for any $\pi$ and $C_2 \ge 0$, $\erroff(\pi) = 0$ when either 
(1)  $\Cscr(d_\pi;\mu,\Fcal,\pi) \le C_2$, 
or, (2) $\Delta f_{\pi} - \gamma \Pcal^\pi \Delta f_{\pi} \equiv 0$. 
\end{corollary}

\paragraph{Adaptive guarantees by algorithm design.} As mentioned above, Theorem~\ref{thm:infothebd2} implicitly selects the best bias-variance decomposition through the best choice of $C_2$ in hindsight, with this decomposition being purely a proof technique, not an knob in the algorithm. In contrast, many prior approaches~\citep{liu2020provably,fujimoto2019off,kumar2019stabilizing} employ explicit thresholds to control density ratios in their algorithms, which makes the tradeoff a hyperparameter in their algorithms. Since choosing hyperparameters is particularly challenging in offline RL, where even policy evaluation can be unreliable, this novel axis of adaptivity is an extremely desirable property of our approach.

\noindent\paragraph{Comparison to guarantees in the exploratory setting.} %
When a dataset with full coverage is given, classical analyses provide near-optimality guarantees that compete with the optimal policy $\pi^\star$  with a polynomial sample complexity, when $\pi^\star\in \Pi$ and both realizability and completeness hold for $\Fcal$; see~\citep{antos2008learning} for a representative analysis. As mentioned earlier, such analysis often requires boundedness of $\|\nu/\mu\|_\infty$ \emph{for all admissible distributions $\nu \in \{d_\pi: \pi \in \Pi\}$}. On the other hand, it is easily seen that we can compete with $\pi^\star$ under much weaker conditions.
\begin{corollary}[Competing with optimal policy] \label{cor:opt}
Under conditions of Theorem~\ref{thm:infothebd2}, if
$\Cscr(d_{\pi^\star};\mu,\Fcal,\pi^\star) \le C_2$, we have 
\begin{equation}
J(\pi^\star) - J(\pihat) \leq ~ \Ocal\left( \frac{\Vmax \sqrt{C_2}}{1 - \gamma} \sqrt{\frac{\log \frac{|\Fcal||\Pi|}{\delta}}{n}} + \frac{\sqrt{C_2 (\varepsilon_{\Fcal,\Fcal} + \varepsilon_\Fcal)}}{1 - \gamma} \right).
\end{equation}
\end{corollary}
Notably, these milder coverage assumptions in Corollaries~\ref{cor:DR} and~\ref{cor:opt} provide offline RL counterparts to the benefits of  policy-gradient-style methods with \textit{online} access to the environment~\citep{kakade2002approximately, scherrer2014approximate,agarwal2019theory}. 

\paragraph{Comparison with~\citet{liu2020provably}.} The closest prior result to our work is that of~\citep{liu2020provably}, who develop a pessimistic estimator that truncates Bellman backups from state-action pairs infrequently visited by $\mu$, and analyzes the resulting pessimistic policy and value iteration algorithms under general function approximation. For truncating Bellman backups, however, their work requires estimating the state-action distribution of data, which can be challenging in high-dimensions and they incur additional errors from density estimation which we avoid. Further, their algorithms only compete with policies $\pi$ where $\|d_\pi/\mu\|_\infty$ is bounded instead of the more general result that we provide, and makes their results vacuous in a linear MDP setting under typical feature coverage assumptions.

\paragraph{Safe Policy Improvement.} Some prior works~\citep[e.g.][]{laroche2019safe,liu2020provably} discuss the scenario where the dataset $\Dcal$ is collected with a behavior policy $\pi_b$ with $\mu=d_{\pi_b}$, and demonstrate that their algorithms always return a policy competitive with $\pi_b$. In our setup, this is straightforward as $d_{\pi_b}$ is always covered, as shown next.
\begin{corollary}[Bounded degradation from behavior policy]
\label{cor:gdvs}
Under conditions of Theorem~\ref{thm:infothebd2}, if $\mu=d_{\pi_b}$ for some policy $\pi_b\in\Pi$, we have 
\begin{align}
J(\pi_b) - J(\pihat) \leq \Ocal\left( \frac{\Vmax}{1 - \gamma} \sqrt{\frac{\log \frac{|\Fcal||\Pi|}{\delta}}{n}} + \frac{\sqrt{\varepsilon_{\Fcal,\Fcal} + \varepsilon_\Fcal}}{1 - \gamma}\right).
\end{align}
\end{corollary}

\begin{proof}[Proof sketch of Theorem~\ref{thm:infothebd2}]
We now briefly describe the core ideas in the proof. More detailed arguments are deferred to the full proof in Appendix~\ref{sec:info_theo_proof}.

The key to prove Theorem~\ref{thm:infothebd2} is to translate the $J(\pi) - J(\pihat)$ to the Bellman error of value functions in the version space $\Fcal_{\pi,\varepsilon}$. Our main proving strategies are as follows:
\begin{enumerate}[leftmargin=*,topsep=1pt,parsep=1pt,itemsep=1pt,partopsep=1pt]
\item As the selection of $\varepsilon = \varepsilon_r$ ensures the accurate estimation $Q^\pi$ is contained in the version space $\Fcal_{\pihat,\varepsilon}$ for any $\pi$, we can obtain $J(\pi) - J(\pihat) \leq \max_{f \in \Fcal_{\pi,\varepsilon}}f(s_0,\pi) - \min_{f \in \Fcal_{\pihat,\varepsilon}}f(s_0,\pihat) + \text{approximation error}$.
\item By the optimality of $\pihat$, we have $\min_{f \in \Fcal_{\pihat,\varepsilon}}f(s_0,\pihat) \geq \min_{f \in \Fcal_{\pi,\varepsilon}}f(s_0,\pi)$. This indicates that $J(\pi) - J(\pihat) \leq \max_{f \in \Fcal_{\pi,\varepsilon}}f(s_0,\pi) - \min_{f \in \Fcal_{\pi,\varepsilon}}f(s_0,\pi) + \text{approximation error}$.
\item By using a standard telescoping argument (e.g.,~\citep[Lemma 1]{xie2020q}), $\max_{f \in \Fcal_{\pi,\varepsilon}}f(s_0,\pi) - \min_{f \in \Fcal_{\pi,\varepsilon}}f(s_0,\pi)$ can be upper bounded by the Bellman error of $\argmax_{f \in \Fcal_{\pi,\varepsilon}}f(s_0,\pi)$ and $\argmin_{f \in \Fcal_{\pi,\varepsilon}}f(s_0,\pi)$ over distribution $d_\pi$.
\end{enumerate}
After combining all the three steps above together and considering the distribution shift effect, we complete the proof.
\end{proof}

\subsection{Results for Linear Function Approximation}
\label{sec:constrdpi_linear}

Here we perform a case study in linear function approximation. %
We will show that our results---when instantiated under linear function approximation (with realizability and completeness assumptions)---automatically provides state-of-the-art guarantees, improving over existing results specialized to this setting by a factor of $\Ocal(d)$~\citep{jin2020pessimism} when the action space is finite and small. 

We recall the linear function approximation setup (we set $R_{\max} = 1$ and $V_{\max} = \tfrac{1}{1-\gamma}$ for consistency with literature).\footnote{\label{ft:errata} In an earlier version of the paper we used the greedy policy class and made a mistake in the proof; correcting the proof would yield an additional $|\Acal|$ dependence in the bound. Here we fixed this issue by changing to the softmax policy class and leveraging covering arguments and concentration bounds from \citet{zanette2021provable} and \citet{cheng2022adversarially}, respectively.}
\begin{definition}[Linear Function Approximation]
\label{def:linearmdp}
Let $\phi:\Scal \times \Acal \to \RR^d$ be a feature mapping. Without loss of generality, we assume $\|\phi(s,a)\|_2 \leq 1$, $\forall (s,a) \in \Scal \times \Acal$. We define the value-function class $\Fcal_\Phi$ as $\Fcal_\Phi \coloneqq \{\phi(\cdot,\cdot)^\T \theta: \theta \in \RR^d, \| \theta \|_2 \leq L_1, \phi(\cdot,\cdot)^\T \theta \in [0,\Vmax]\}$, and the policy class $\Pi_\Phi$ is the softmax policies of $\Pi_\Phi \coloneqq \{\phi(\cdot,\cdot)^\T \theta: \theta \in \RR^d, \| \theta \|_2 \leq L_2\}$.
\end{definition}
\begin{assumption}[Realizability and Completeness]
\label{asm:lin_real_comp}
$\varepsilon_{\Fcal,\Fcal} = \varepsilon_\Fcal = 0$.
\end{assumption}

Note that, when the feature mapping $\phi(\cdot,\cdot)$ is the one induced by the linear MDP \citep{jin2020provably}, it automatically ensures that $\Fcal_\Phi$ satisfies Assumptions~\ref{asm:lin_real_comp}. %
We also highlight that 
the standard linear function approximation or linear MDP setup does \textit{not} entail all the assumptions needed by~\citet{liu2020provably} as mentioned earlier. 
To be consistent with the above results of general function classes, below we will present the results using the upper bound $\Vmax$ on the functions in $\Fcal_\Phi$, 
where $\Vmax$ in such a bound can be replaced directly by using the value function radius $L_1$. %
In Appendix~\ref{sec:info_theo_proof_linear}, we show that the policy class $\Pi_\Phi$ contains a near-optimal policy as long as $L_2$ is moderately large.

Below is our main result in the linear function approximation setting. 

\newcommand{\piphi}{\pi_{\Phi}^\star}

\begin{theorem}
\label{thm:linear_bound}
Suppose the value-function class $\Fcal_\Phi$ is a linear function class that satisfies realizability and completeness (Definition~\ref{def:linearmdp} and Assumption~\ref{asm:lin_real_comp}) and $\pihat$ is the output of \Eqref{eq:infotheosol} using value-function class $\Fcal_\Phi$ and policy class $\Pi_\Phi$. If we choose $\varepsilon = (c \Vmax^2 d \log \frac{n L_1 L_2}{\Vmax \delta})/{n}$, then, for any policy $\pi \in \Pi_\Phi$, we have
\begin{align}
J(\pi) - J(\pihat) \leq \Ocal \left( \frac{\Vmax}{1 - \gamma} \sqrt{\frac{d \log \frac{n L_1 L_2}{\Vmax \delta}}{n}} \Ebb_{d_\pi} \left[  \sqrt{\phi(s,a)^\T \Sigma_\Dcal^{-1} \phi(s,a)} \right] \right),
\end{align}
where $c$ is an absolute constant, and $\Sigma_{\Dcal} \coloneqq \E_\Dcal \left[\phi(s,a) \phi(s,a)^\T \right]$.
\end{theorem}
The detailed proof of Theorem~\ref{thm:linear_bound} is provided in Appendix~\ref{sec:info_theo_proof_linear}. Our guarantee is structurally very similar to that of~\citet[Theorem 4.4]{jin2020pessimism}, except that we only need linear function approximation with realizability and completeness assumptions, and they consider the finite-horizon linear MDP setting. 
If we translate their result to the discounted setting by setting $H=\Ocal(1/(1-\gamma))$, we enjoy a net improvement of order $\Ocal(d)$ in sample complexity when the action space is finite and small. To make it concrete, that is $\Ocal(\sqrt{\nicefrac{d^2\log(dn/\delta)}{n}})$ vs.~$\Ocal(\sqrt{\nicefrac{d\log(n L_1 L_2/\Vmax \delta)}{n}})$ error bounds.
Careful readers may notice that \citet{jin2020pessimism} compete with $\pi^\star$ whereas we can only compete with the best policy in $\Pi_\Phi$, but this can be easily addressed without changing the order of our error bound: in Appendix~\ref{sec:info_theo_proof_linear} we show that  %
selecting a large policy radius (e.g., $L_2 = \nicefrac{L_1 \log|\Acal| \sqrt{n}}{\Vmax}$) ensures the existence of a good policy in $\Pi_\Phi$, denoted as $\piphi$, whose suboptimality is no larger than the bound in Theorem~\ref{thm:linear_bound}, and hence we can compete with $\pi^\star$ and the additive error %
will be absorbed by the big-Oh.\footnote{This introduces another subtle difference between our bound and \citet{jin2020pessimism} when both compete with $\pi^\star$: \citet{jin2020pessimism} requires coverage of $\pi^\star$, whereas we require coverage of $\piphi$.} %
Our bound also shows that having a full-rank $\Sigma_{\Dcal}$ (which is ensured by a full-rank covariance under $\mu$) is sufficient for consistent offline RL in linear function approximation. Crucially, the full-rank covariance is an easily checkable condition on data, as opposed to unverifiable concentrability assumptions. As a caveat, our results do not imply a computationally efficient algorithm, as a na\"ive implementation involves evaluating each policy pessimistically to pick the best. We discuss a computationally efficient adaptation of our approach in the next section.

\section{Practical Algorithm --- Regularized Offline Policy Optimization}
\label{sec:regulpi}

A major challenge using the proposed algorithm in Section~\ref{sec:constrdpi} in practice is that searching the policy with the best pessimistic evaluation over the policy space $\Pi$ is not computationally tractable. In this section, we present a practical algorithm that is computationally efficient assuming access to a (regularized) loss minimization oracle over the value function class $\Fcal$, and also comes with rigorous theoretical guarantees.

\begin{algorithm*}[th]
\caption{PSPI: Pessimistic Soft Policy Iteration}
\label{alg:reg_pi}
{\bfseries Input:} Batch data $\Dcal$, regularization coefficient $\lambda$.
\begin{algorithmic}[1]
\State Initialize policy $\pi_1$ as the uniform policy.
\For{$t = 1,2,\dotsc,T$}
\State Obtain the pessimistic estimation for $\pi_t$ as $f_t$, %
\label{lin:reg_pes_pe}
\begin{align}
\label{eq:reg_pes_pe}
f_t \leftarrow \argmin_{f \in \Fcal} \left(f(s_0,\pi_t) + \lambda \Ecal(f,\pi_t;\Dcal)\right),
\end{align}
where $\Ecal(f,\pi_t;\Dcal)$ is defined in \Eqref{eq:defmsbope}.
\State Calculate $\pi_{t + 1}$ by, %
\begin{align}
\label{eq:calculate_pi}
\pi_{t + 1}(a|s) \propto \pi_{t}(a|s) \exp \left( \eta f_t(s,a) \right),~\forall s,a \in \Scal \times \Acal.
\end{align}
\EndFor
\State Output $\pibar \coloneqq \unif(\pi_{[1:T]})$. \algocmt{uniformly mix $\pi_1, \ldots, \pi_T$ at the trajectory level}
\end{algorithmic}
\end{algorithm*}

Our practical algorithm is summarized in Algorithm~\ref{alg:reg_pi}. It has three key differences from the information-theoretic version in \Eqref{eq:infotheosol}:
\begin{enumerate}[leftmargin=*,topsep=1pt,parsep=1pt,itemsep=1pt,partopsep=1pt]
\item The pessimistic policy evaluation is now performed via regularization (Line~\ref{lin:reg_pes_pe}) instead of constrained optimization. 
\item Instead of searching over an explicit policy space $\Pi$, %
we search over a policy class implicitly induced from $\Fcal$ (defined in \Eqref{eq:policy-spi}) and therefore no longer have a policy class independent of $\Fcal$ separately, which is a common practice in API-style algorithms~\citep{munos2003error,antos2008learning, lazaric2012finite}.
\item We optimize the policy using mirror descent updates, which yields computationally tractable optimization over the implicit policy class. This property has been leveraged in many prior works, although typically in online RL settings~\citep{even2009online, agarwal2019theory,geist2019theory, cai2020provably, shani2020optimistic}. 
\end{enumerate} 

Note that the use of a specific policy class above can be relaxed 
if a stronger structural assumption is made on the MDP (e.g., linear MDPs~\citep{jin2020provably,jin2020pessimism}).

\subsection{Analysis of Algorithm~\ref{alg:reg_pi}}

We now provide the analysis of Algorithm~\ref{alg:reg_pi} in this section. For ease of presentation, we formally define the implicit policy class for this section:
\begin{equation}
\label{eq:policy-spi}
    \Pi_{\textrm{SPI}} := \{ \pi'(\cdot|s) \propto \exp(\eta \sum_{i=1}^t f^{(t)}(s, \cdot)): 1\le t \le T, f^{(1)}, \ldots, f^{(i)} \in \Fcal\},
\end{equation}
which is the natural policy class for soft policy-iteration approaches.
The following theorem describes the performance guarantee of $\pibar$.
\begin{theorem}
\label{thm:opterr_regpi}
Let $\lambda = \sqrt[3]{\nicefrac{\Vmax}{(1 - \gamma)^2\varepsilon_r^2}}$ with $\varepsilon_r$ in \Eqref{eq:def_varepsilon_r}, $\eta = \sqrt{\frac{\log|\Acal|}{2 \Vmax^2 T}}$, and $\pibar$ be obtained from Algorithm~\ref{alg:reg_pi}. For any policy $\pi:\Scal\to\Delta(\Acal)$ we wish to compete with, suppose Assumptions~\ref{asm:relz2} and~\ref{asm:comp_restat} hold with respect to the policy class $\Pi_{\textrm{SPI}}\cup\{\pi\}$. Then, for any constant $C_{2} \geq 1$, we have with probability at least $1-\delta$, 
\begin{align}
&~ J(\pi) - J(\pibar) 
\\
\leq &~ \underbrace{\Ocal\left( \sqrt{C_{2}} \left(\frac{\sqrt{\varepsilon_{\Fcal,\Fcal} + \varepsilon_\Fcal}}{1 - \gamma} + \frac{\Vmax}{1 - \gamma}  \sqrt[3]{\frac{T\log \frac{|\Fcal|}{\delta}}{n}} + \sqrt[3]{\frac{\Vmax \varepsilon_\Fcal}{(1 - \gamma)^2}} \right)  \right)}_{\text{$\erron(\pi)$: on-support error}} + \underbrace{\Ocal\left(\frac{\Vmax}{1 - \gamma}\sqrt{\frac{\log|\Acal|}{T}}\right)}_{\text{optimization error}}
\\
&~ + \underbrace{\frac{1}{T} \sum_{t = 1}^{T} \left(\min_{\nu:\Cscr(\nu;\mu,\Fcal,\pi_t) \leq C_{2}} \left|\sum_{(s,a) \in \Scal\times\Acal} \frac{(d_{\pi}\setminus\nu)(s,a) \left[f_t(s,a) - (\Tcal^{\pi_t} f_t)(s,a)\right]}{1 - \gamma} \right| \right)}_{\text{$\erroff(\pi)$: off-support error}},
\end{align}
where $\Cscr(\nu;\mu,\Fcal,\pi_t)$ is defined in Definition~\ref{def:concenbddist}, $(d_{\pi}\setminus\nu)(s,a) \coloneqq \max(d_\pi(s,a) - \nu(s,a),0)$.%
\end{theorem}
We provide a proof sketch of Theorem~\ref{thm:opterr_regpi} at the end of this section, and defer the full proof to Appendix~\ref{sec:prac_algo_proof}. We now make a few remarks about the results in Theorem~\ref{thm:opterr_regpi}.
\paragraph{Measurement of distribution shift effect.}
Compared with the information-theoretical result (provided in Theorem~\ref{thm:infothebd2}), the measurement of distribution shift in Theorem~\ref{thm:opterr_regpi} depends on the optimization trajectory. That is, it measures the distance between two distribution $\nu$ and $\mu$ by $\Cscr(\nu;\mu,\Fcal,\pi_t)$ ($\pi_{[1:T]}$ is the sequence of policies produced by the algorithm) whereas Theorem~\ref{thm:infothebd2} uses $\Cscr(\nu;\mu,\Fcal,\pi)$ ($\pi$ is the baseline policy we compete with). We remark that both of these two measurements are weaker then traditional density-ratio definitions (e.g.,~\citep{munos2008finite,chen2019information,xie2020q}) as we demonstrated before, as the dependence of $\Cscr$ on $\pi$ is relatively secondary.

\paragraph{Dependence on $T$.} The number of optimization rounds $T$ affects the bound in two  opposite ways: as $T$ increases, the optimization error term decreases, whereas the second term of the on-support error increases. The latter increase is due to the complexity of the implicit policy class $\Pi$ growing exponentially with $T$, which affects our concentration bounds. To optimize the bound, the optimal choice is $T = \Ocal(n^{2/5})$, leading to an overall $\Ocal(n^{-1/5})$ rate. While such a rate is relatively slow, we remark that the complexity bound of $\Pi$ is conservative, and in certain cases it is possible to obtain much sharper bounds: for example, in linear function approximation (Section~\ref{sec:constrdpi_linear}), %
$\Pi_{\textrm{SPI}}$ are a priori captured by the space of softmax policies, whose complexity has no dependence on $T$ (up to mild logarithmic dependence due to norms). That is, the $\erron(\pi)$ term in Theorem \ref{thm:opterr_regpi} reduces to $\widetilde\Ocal(\frac{\Vmax}{1 - \gamma}  \sqrt[3]{{d}/{n}})$ ($\varepsilon_{\Fcal,\Fcal} = \varepsilon_\Fcal = 0$ in linear function approximation), and yields an overall $\Ocal(n^{-1/3})$ rate.

\paragraph{Bias-variance decomposition.}
Similar to Theorem~\ref{thm:infothebd2}, Theorem~\ref{thm:opterr_regpi} also allows arbitrary decomposition of the error bound into on-support and off-support components by setting the concentrability threshold $C_2$, which serves as a bias-variance tradeoff as before. In fact, the splitting can be done separately for each $\pi_t$ in $1\le t \le T$ and we omit such flexibility for readability.  %
The optimization error does not depend on the splitting. Our performance guarantee naturally adapts to the best possible decomposition as before. %
As in Theorem~\ref{thm:infothebd2}, if the estimation on the off-support region is ``high-quality'', we can further simplify the performance guarantees, but the requirement of ``high-quality'' is different from that of Corollary~\ref{cor:info_theo_dr}.  We make it formal in the following corollary.
\begin{corollary}[``Double Robustness''] \label{cor:DR-prac}
For any $\pi$ and $C_2 \ge 0$, $\erroff(\pi) = 0$ when either 
(1)  $\Cscr(d_\pi;\mu,\Fcal,\pi_t) \le C_2$ for all $t \in [T]$, 
or, (2) $f_{t} - \Tcal^{\pi_t} \Delta f_{t} \equiv 0$ for all $t \in [T]$. 
\end{corollary}
We note that the conditions above depend on the optimization trajectory through their dependence on $\pi_t$, but can be made algorithm-independent by instead asserting the stronger requirement that $\Cscr(d_\pi,\mu,\Fcal,\pi') \leq C_2$ for all $\pi'\in\Pi_{\textrm{SPI}}$ in the first condition. 

\paragraph{Competing with the optimal policy.} As before, we can provide a guarantee for competing with the optimal policy, under coverage assumptions weaker than the typical batch RL literature, albeit slightly stronger than those of Corollary~\ref{cor:opt}. We state the formal result below.
\begin{corollary}[Competing with optimal policy] \label{cor:opt-prac}
Under conditions of Theorem~\ref{thm:opterr_regpi}, if
$\Cscr(d_{\pi^\star};\mu,\Fcal,\pi) \le C_2$ for all $\pi\in\Pi_{\textrm{SPI}}$, we have 
\begin{equation}
J(\pi^\star) - J(\pihat) \leq ~ \Ocal\left( \frac{\Vmax \sqrt{C_2}}{1 - \gamma} \left(\frac{\log \frac{|\Fcal|}{\delta}\log|\Acal|}{n}\right)^{1/5} + \frac{\sqrt{C_2 (\varepsilon_{\Fcal,\Fcal} + \varepsilon_\Fcal)}}{1 - \gamma} \right).
\end{equation}
\end{corollary}
Note that the conditions of Corollary~\ref{cor:opt-prac} are satisfied as before whenever $\|d_{\pi^\star}/\mu\|_\infty \leq C_2$.

\paragraph{Computationally-efficient implementation with linear function approximation}
We remark that our algorithm is computationally efficient when the value-function class $\Fcal$ is linear, that is, $\Fcal \coloneqq \{\phi(\cdot,\cdot)^\T \theta: \theta \in \RR^d\}$. %
In this case, the objective of \Eqref{eq:reg_pes_pe} has a closed-form expression which is quadratic in $\theta$. In addition, under additional matrix invertibility conditions, \Eqref{eq:reg_pes_pe} has a closed-form solution which generalizes  LSTDQ~\citep{lagoudakis2003least,sutton2009fast, dann2014policy}. A similar connection has been made by \citet{antos2008learning}, but our derivation is more general. See Appendix Appendix~\ref{sec:linear_implementation} for further details. %

We conclude the section with a proof sketch showing the key insights used in establishing the proof.
\begin{proof}[\textbf{Proof sketch of Theorem~\ref{thm:opterr_regpi}}]
Our proof constructs a corresponding MDP $\Mcal_t$ for every $f_t, \pi_t$ pair. Each $\Mcal_t$ has the same dynamics as the ground-truth MDP, but chooses a different reward function, such that $f_t$ is the $Q$-function of $\pi_t$ in $\Mcal_t$, $Q^{\pi_t}_{\Mcal_t}$ (we use the subscript of $\Mcal_t$ to denote the corresponding value or operator in MDP $\Mcal_t$). Our proof relies on some key properties of $\Mcal_t$, such as $Q^\pi - \Tcal^\pi_{\Mcal_t} Q^\pi = f_t - \Tcal^{\pi_t}f_t$. We decompose $J(\pi) - J(\bar\pi)$ as follows.
\begin{align}
J(\pi) - J(\bar\pi) \leq &~ \underbrace{\frac{1}{T} \sum_{t = 1}^{T} \left(J_{\Mcal_t}(\pi) - J_{\Mcal_t}(\pi_t) \right)}_{\text{optimization error}} + \underbrace{\frac{1}{T} \sum_{t = 1}^{T} \left(J(\pi) - J_{\Mcal_t}(\pi) \right)}_{\text{controlled by } \|Q^\pi - \Tcal^\pi_{\Mcal_t} Q^\pi\|_{2,d_\pi} = \|f_t - \Tcal^{\pi_t}f_t\|_{2,d_\pi}}
\\
&~ + \text{approximation/statistical errors}.
\end{align}
The proof is completed by bounding $\|f_t - \Tcal^{\pi_t}f_t\|_{2,d_\pi}$ on both on-support and off-support regions.
\end{proof}
\section{Conclusions}
\label{sec:conclusion}

This paper investigates sample-efficient offline reinforcement learning without data coverage assumptions (e.g., concentrability). To achieve that goal, our paper contributes several crucial improvements to the literature. We introduce the concept of Bellman-consistent pessimism. It enables the sample-efficient guarantees with only the Bellman-completeness assumption which is standard in the exploratory setting, whereas the point-wise/bonus-based pessimism popularly adopted in the literature usually requires stronger and/or extra assumptions. Algorithmically, we demonstrate how to implicitly infer a policy value lower bound through a version space and provide a tractable implementation. A particularly important aspect of our results is the ability to adapt to the best bias-variance tradeoff in the hindsight, which no prior algorithms achieve to the best of our knowledge. When applying our results in linear function approximation, we attain an $\Ocal(d)$ improvement in sample complexity, compared with the best-known recent work of offline RL in linear MDPs, whenever the action space is finite and small.

As of limitations and future work, the sample complexity of our practical algorithm is worse than that of the information-theoretic approach, and it will be interesting to close this gap. Another future direction is to empirically evaluate PSPI on benchmarks and compare it to existing approaches. 
\section*{Acknowledgment}
Part of this work was carried out while TX and AA worked at Microsoft Research. NJ acknowledges funding support from the ARL Cooperative Agreement W911NF-17-2-0196, NSF IIS-2112471, and Adobe Data Science Research Award. The authors also thank Jinglin Chen for pointing out a mistake in the conference version of the paper (see Footnote~\ref{ft:errata}).

\bibliographystyle{plainnat}
\bibliography{ref}

\clearpage
\allowdisplaybreaks

\appendix
\onecolumn

\begin{center}
{\LARGE Appendix}
\end{center}

\section{Estimating Mean-Squared Bellman Error}
\label{sec:plceval}

We provide theoretical properties of $\Ecal(f,\pi;\Dcal)$ (defined in \Eqref{eq:defmsbope}), when using it to bound Bellman error.
Over this section, we use $\mu\times(\Pcal,R)$ to denote the joint distribution of $(s,a,r,s')$.

All of our results in the appendix strongly depends one following two constants 
\begin{enumerate}[leftmargin=*,topsep=1pt,parsep=1pt,itemsep=1pt,partopsep=1pt]
\item $\varepsilon_r$ --- For any $\pi \in \Pi$, if $f_\pi$ is the ``best estimation'' of $Q^\pi$ in $\Fcal$ (formal definition of $f_\pi$: \Eqref{eq:deff_pi} for general function approximation, $Q^\pi$ for Linear function approximation), then $\Ecal(f_\pi,\pi;\Dcal) \leq \varepsilon_r$.
\item $\varepsilon_b$ --- For any function $f \in \Fcal$ and any $\pi \in \Pi$, if $\Ecal(f,\pi;\Dcal) \leq \varepsilon_r$, then $\| f - \Tcal^\pi f\|_{2,\mu}^2 \leq \varepsilon_b$.
\end{enumerate}
The detailed discussion about $\varepsilon_r$ and $\varepsilon_b$ is provided in Appendix \ref{sec:PEresults_generalFA} (general function approximation) and and Appendix \ref{sec:PEresults_linearMDP} (linear Linear function approximation).

\subsection{Results for General Function Approximation}
\label{sec:PEresults_generalFA}

This section summarizes the results regarding $\Ecal(f,\pi;\Dcal)$, we defer the full proof of Theorem~\ref{thm:version_space} and Theorem~\ref{thm:mspo2be} to Appendix~\ref{appx:peproofs}.

\begin{theorem}
\label{thm:version_space}
For any $\pi \in \Pi$, let $f_\pi$ be defined as follows,
\begin{align}
\label{eq:deff_pi}
f_\pi \coloneqq &~ \argmin_{f \in \Fcal}\sup_{\text{admissible }\nu} \left\|f - \Tcal^\pi f\right\|_{2,\nu}^2.
\end{align}
Then, for $\Ecal(f_\pi,\pi;\Dcal)$ (defined in \Eqref{eq:defmsbope}), we have
\begin{align}
\label{eq:upbdespr}
\Ecal(f_\pi,\pi;\Dcal) \leq \frac{139 \Vmax^2 \log \frac{|\Fcal||\Pi|}{\delta}}{n} + 39 \varepsilon_{\Fcal} \eqqcolon \varepsilon_r.
\end{align}
\end{theorem}

We now show that $\Ecal(f,\pi;\Dcal)$ could effectively estimate $\|f - \Tcal^\pi f \|_{2,\mu}^2$.
\begin{theorem}
\label{thm:mspo2be}
For any $\pi \in \Pi$, $f \in \Fcal$, and any $\varepsilon > 0$, if $\Ecal(f,\pi;\Dcal) \leq \varepsilon$, then, 
\begin{align}
\label{eq:upbdmsbe_ori}
\| f - \Tcal^\pi f\|_{2,\mu} \leq \Vmax \sqrt{\frac{ 231 \log \frac{|\Fcal||\Pi|}{\delta}}{n}} + \sqrt{\varepsilon_{\Fcal,\Fcal}} + \sqrt{\varepsilon_{\Fcal,\Fcal} + \varepsilon}.
\end{align}
\end{theorem}
We also define $\sqrt{\varepsilon_b}$ when setting $\varepsilon = \varepsilon_r$ in \Eqref{eq:upbdmsbe_ori}, i.e.,
\begin{align}
\label{eq:upbdmsbe}
\sqrt{\varepsilon_b} \coloneqq \Vmax \sqrt{\frac{ 231 \log \frac{|\Fcal||\Pi|}{\delta}}{n}} + \sqrt{\varepsilon_{\Fcal,\Fcal}} + \sqrt{\varepsilon_{\Fcal,\Fcal} + \varepsilon_r}.
\end{align}

\subsubsection{Complementary Lemmas for General Function Approximation}

We first provide some complementary lemmas that used in our detailed proofs of Theorem~\ref{thm:version_space} and Theorem~\ref{thm:mspo2be}.

\begin{lemma}
\label{lem:fastrate}
For any $f_1, f_2 \in \Fcal$ and $\pi \in \Pi$, w.p.~$1 - \delta$,
\begin{align}
\left| \left\|f_1(s,a) - (\Tcal^\pi f_2)(s,a)\right\|_{2,\mu} - \left\|f_1(s,a) - (\Tcal^\pi f_2)(s,a)\right\|_{2,\Dcal} \right| \leq \sqrt{\frac{5 \Vmax^2 \log \frac{|\Fcal||\Pi|}{\delta}}{n}}.
\end{align}
\end{lemma}
\begin{proof}[\bf\em Proof of Lemma~\ref{lem:fastrate}]
The proof of this lemma follows the similar proving strategy of \citep[Lemma 10]{xie2021batch}.
We first apply Bernstein’s inequality with a union bound over $\Fcal \times \Fcal$, and obtain: w.p.~$1 - \delta$, for any $f_1, f_2 \in \Fcal$,
\begin{align}
&~ \left| \left\|f_1(s,a) - (\Tcal^\pi f_2)(s,a)\right\|_{2,\mu}^2 - \left\|f_1(s,a) - (\Tcal^\pi f_2)(s,a)\right\|_{2,\Dcal}^2 \right|
\\
= &~ \left|\E_{\mu\times(\Pcal,R)} \left[ \left( f_1(s,a) - (\Tcal^\pi f_2)(s,a) \right)^2 \right] - \frac{1}{n} \sum_{(s,a,r,s') \in \Dcal} \left( f_1(s,a) - (\Tcal^\pi f_2)(s,a) \right)^2 \right|
\\
\leq &~ \sqrt{\frac{4\V_\mu \left[ \left( f_1(s,a) - (\Tcal^\pi f_2)(s,a) \right)^2 \right] \log \frac{|\Fcal||\Pi|}{\delta}}{n}} + \frac{2 \Vmax^2 \log \frac{|\Fcal||\Pi|}{\delta}}{3n}
\\
\label{eq:square2normbound}
\leq &~ \left\|f_1(s,a) - (\Tcal^\pi f_2)(s,a)\right\|_{2,\mu} \sqrt{\frac{4 \Vmax^2 \log \frac{|\Fcal||\Pi|}{\delta}}{n}} + \frac{2 \Vmax^2 \log \frac{|\Fcal||\Pi|}{\delta}}{3n},
\end{align}
where the last inequality is obtained by the following argument,
\begin{align}
&~ \V_{\mu\times(\Pcal,R)} \left[ \left( f_1(s,a) - (\Tcal^\pi f_2)(s,a) \right)^2 \right]
\\
\leq &~ \E_{\mu\times(\Pcal,R)} \left[ \left( f_1(s,a) - (\Tcal^\pi f_2)(s,a) \right)^4 \right]
\\
\leq &~ \Vmax^2 \E_{\mu\times(\Pcal,R)} \left[ \left( f_1(s,a) - (\Tcal^\pi f_2)(s,a) \right)^2 \right],
\end{align}
where the last inequality follows from $|f_1(s,a) - (\Tcal^\pi f_2)(s,a)| \leq \Vmax, ~\forall (s,a,r,s')$.

By the fact of $|a - b|^2 \leq |a^2 - b^2|$, we know
\begin{align}
&~ \left| \left\|f_1(s,a) - (\Tcal^\pi f_2)(s,a)\right\|_{2,\mu} - \left\|f_1(s,a) - (\Tcal^\pi f_2)(s,a)\right\|_{2,\Dcal} \right|
\\
\leq &~ \sqrt{\left| \left\|f_1(s,a) - (\Tcal^\pi f_2)(s,a)\right\|_{2,\mu}^2 - \left\|f_1(s,a) - (\Tcal^\pi f_2)(s,a)\right\|_{2,\Dcal}^2 \right|}
\\
\leq &~ \sqrt{\left\|f_1(s,a) - (\Tcal^\pi f_2)(s,a)\right\|_{2,\mu} \sqrt{\frac{4 \Vmax^2 \log \frac{|\Fcal||\Pi|}{\delta}}{n}} + \frac{2 \Vmax^2 \log \frac{|\Fcal||\Pi|}{\delta}}{3n}}
\tag{by \Eqref{eq:square2normbound}}
\\
\label{eq:2normbd1}
\leq &~ \sqrt{\left\|f_1(s,a) - (\Tcal^\pi f_2)(s,a)\right\|_{2,\mu}} \sqrt[4]{\frac{4 \Vmax^2 \log \frac{|\Fcal||\Pi|}{\delta}}{n}} + \sqrt{\frac{2 \Vmax^2 \log \frac{|\Fcal||\Pi|}{\delta}}{3n}}.
\end{align}

On the other hand,
\begin{align}
&~ \left| \left\|f_1(s,a) - (\Tcal^\pi f_2)(s,a)\right\|_{2,\mu} - \left\|f_1(s,a) - (\Tcal^\pi f_2)(s,a)\right\|_{2,\Dcal} \right|
\\
\leq &~ \frac{\left| \left\|f_1(s,a) - (\Tcal^\pi f_2)(s,a)\right\|_{2,\mu}^2 - \left\|f_1(s,a) - (\Tcal^\pi f_2)(s,a)\right\|_{2,\Dcal}^2 \right|}{\left\|f_1(s,a) - (\Tcal^\pi f_2)(s,a)\right\|_{2,\mu}}
\\
\leq &~ \frac{\left\|f_1(s,a) - (\Tcal^\pi f_2)(s,a)\right\|_{2,\mu} \sqrt{\frac{4 \Vmax^2 \log \frac{|\Fcal||\Pi|}{\delta}}{n}} + \frac{2 \Vmax^2 \log \frac{|\Fcal||\Pi|}{\delta}}{3n}}{\left\|f_1(s,a) - (\Tcal^\pi f_2)(s,a)\right\|_{2,\mu}}
\tag{by \Eqref{eq:square2normbound}}
\\
\label{eq:2normbd2}
\leq &~ \sqrt{\frac{4 \Vmax^2 \log \frac{|\Fcal||\Pi|}{\delta}}{n}} + \frac{2 \Vmax^2 \log \frac{|\Fcal||\Pi|}{\delta}}{3n \left\|f_1(s,a) - (\Tcal^\pi f_2)(s,a)\right\|_{2,\mu}}.
\end{align}

Combining \Eqref{eq:2normbd1} and \Eqref{eq:2normbd2}, we obtain
\begin{align}
&~ \left| \left\|f_1(s,a) - (\Tcal^\pi f_2)(s,a)\right\|_{2,\mu} - \left\|f_1(s,a) - (\Tcal^\pi f_2)(s,a)\right\|_{2,\Dcal} \right|
\\
\leq &~ \min\Bigg( \sqrt{\underbrace{\left\|f_1(s,a) - (\Tcal^\pi f_2)(s,a)\right\|_{2,\mu}}_{\eqqcolon g \in \RR}} \sqrt[4]{\frac{4 \Vmax^2 \log \frac{|\Fcal||\Pi|}{\delta}}{n}} + \sqrt{\frac{2 \Vmax^2 \log \frac{|\Fcal||\Pi|}{\delta}}{3n}},
\\
&~ \sqrt{\frac{4 \Vmax^2 \log \frac{|\Fcal||\Pi|}{\delta}}{n}} + \frac{2 \Vmax^2 \log \frac{|\Fcal||\Pi|}{\delta}}{3n \left\|f_1(s,a) - (\Tcal^\pi f_2)(s,a)\right\|_{2,\mu}}\Bigg)
\\
\leq &~ \max_{g \in \RR} \min\left( \sqrt{g} \sqrt[4]{\frac{4 \Vmax^2 \log \frac{|\Fcal||\Pi|}{\delta}}{n}} + \sqrt{\frac{2 \Vmax^2 \log \frac{|\Fcal||\Pi|}{\delta}}{3n}},
\sqrt{\frac{4 \Vmax^2 \log \frac{|\Fcal||\Pi|}{\delta}}{n}} + \frac{2 \Vmax^2 \log \frac{|\Fcal||\Pi|}{\delta}}{3n g}\right)
\\
\leq &~ \min_{g \in \RR} \max\left( \sqrt{g} \sqrt[4]{\frac{4 \Vmax^2 \log \frac{|\Fcal||\Pi|}{\delta}}{n}} + \sqrt{\frac{2 \Vmax^2 \log \frac{|\Fcal||\Pi|}{\delta}}{3n}},
\sqrt{\frac{4 \Vmax^2 \log \frac{|\Fcal||\Pi|}{\delta}}{n}} + \frac{2 \Vmax^2 \log \frac{|\Fcal||\Pi|}{\delta}}{3n g}\right)
\\
\overset{\text{(a)}}{\leq} &~ \max\left(\sqrt{\frac{4 \Vmax^2 \log \frac{|\Fcal||\Pi|}{\delta}}{n}} + \sqrt{\frac{2 \Vmax^2 \log \frac{|\Fcal||\Pi|}{\delta}}{3n}},
\sqrt{\frac{4 \Vmax^2 \log \frac{|\Fcal||\Pi|}{\delta}}{n}} + \sqrt{\frac{\Vmax^2 \log \frac{|\Fcal||\Pi|}{\delta}}{9 n}}\right)
\\
\leq &~ \sqrt{\frac{5 \Vmax^2 \log \frac{|\Fcal||\Pi|}{\delta}}{n}},
\end{align}
where (a) is by selecting $g = \sqrt{\nicefrac{4 \Vmax^2 \log \frac{|\Fcal||\Pi|}{\delta}}{n}}$. This completes the proof.
\end{proof}

\begin{lemma}
\label{lem:bernstein_general}
For any $f, g_1, g_2 \in \Fcal$ and $\pi \in \Pi$, we have
\begin{align}
&~ \bigg| \| g_1 - \Tcal^\pi f\|_{2,\mu}^2 - \left\| g_2 - \Tcal^\pi f\right\|_{2,\mu}^2
\\
&~ - \frac{1}{n} \sum_{(s,a,r,s') \in \Dcal} \left(g_1(s,a) - r - \gamma f(s',\pi) \right)^2 + \frac{1}{n} \sum_{(s,a,r,s') \in \Dcal} \left(g_2(s,a)  - r - \gamma f(s',\pi) \right)^2 \bigg|
\\
\leq &~ 4 \Vmax \|g_1 - g_2\|_{2,\mu}\sqrt{\frac{ \log \frac{|\Fcal||\Pi|}{\delta}}{n}} + \frac{2 \Vmax^2 \log \frac{|\Fcal||\Pi|}{\delta}}{3n}.
\end{align}
\end{lemma}
\begin{proof}[\bf\em Proof of Lemma~\ref{lem:bernstein_general}]
By a standard calculation,
\begin{align}
&~ \frac{1}{n} \sum_{(s,a,r,s') \in \Dcal} \left(g_1(s,a) - r - \gamma f(s',\pi) \right)^2 - \frac{1}{n} \sum_{(s,a,r,s') \in \Dcal} \left(g_2(s,a)  - r - \gamma f(s',\pi) \right)^2
\\
= &~ \frac{1}{n} \sum_{(s,a,r,s') \in \Dcal} \left( \left(g_1(s,a) - r - \gamma f(s',\pi) \right)^2 - \left(g_2(s,a)  - r - \gamma f(s',\pi) \right)^2 \right)
\\
\label{eq:empmsbo}
= &~ \frac{1}{n} \sum_{(s,a,r,s') \in \Dcal} \left( \left(g_1(s,a) - g_2(s,a) \right) \left(g_1(s,a) + g_2(s,a)  - 2 r - 2 \gamma f(s',\pi) \right) \right).
\end{align}

Similarly, we also have the following fact,
\begin{align}
&~ \E_{\mu\times(\Pcal,R)} \left[\left(g_1(s,a) - r - \gamma f(s',\pi) \right)^2\right] - \E_{\mu\times(\Pcal,R)} \left[\left(g_2(s,a) - r - \gamma f(s',\pi) \right)^2 \right]
\\
\overset{\text{(a)}}{=} &~ \E_{\mu\times(\Pcal,R)} \left[ \left(g_1(s,a) - g_2(s,a) \right) \left(g_1(s,a) + g_2(s,a)  - 2 r - 2 \gamma f(s',\pi) \right) \right]
\\
= &~ \E_\mu \left[ \E \left[ \left(g_1(s,a) - g_2(s,a) \right) \left(g_1(s,a) + g_2(s,a)  - 2 r - 2 \gamma f(s',\pi) \right) \middle| s,a \right] \right]
\\
\label{eq:popmsbo_middle}
= &~ \E_\mu \left[  \left(g_1(s,a) - g_2(s,a) \right) \left(g_1(s,a) + g_2(s,a)  - 2 \left(\Tcal^\pi f\right)(s,a) \right) \right]
\\
\label{eq:popmsbo}
\overset{\text{(b)}}{=} &~ \E_{\mu} \left[\left(g_1(s,a) - \left(\Tcal^\pi f\right)(s,a) \right)^2\right] - \E_{\mu} \left[\left(g_2(s,a) - \left(\Tcal^\pi f\right)(s,a) \right)^2 \right],
\end{align}
where (a) and (b) follow from the similar argument to \Eqref{eq:empmsbo}.

By using \Eqref{eq:empmsbo} and \Eqref{eq:popmsbo}, we know
\begin{align}
&~ \E_{\mu\times(\Pcal,R)} \left[ \frac{1}{n} \sum_{(s,a,r,s') \in \Dcal} \left(g_1(s,a) - r - \gamma f(s',\pi) \right)^2 - \frac{1}{n} \sum_{(s,a,r,s') \in \Dcal} \left(g_2(s,a)  - r - \gamma f(s',\pi) \right)^2 \right]
\\
= &~\E_{\mu} \left[\left(g_1(s,a) - \left(\Tcal^\pi f\right)(s,a) \right)^2\right] - \E_{\mu} \left[\left(g_2(s,a) - \left(\Tcal^\pi f\right)(s,a) \right)^2 \right].
\end{align}
Then, 
\begin{align}
&~ \bigg| \E_{\mu} \left[\left(g_1(s,a) - \left(\Tcal^\pi f\right)(s,a) \right)^2\right] - \E_{\mu} \left[\left(g_2(s,a) - \left(\Tcal^\pi f\right)(s,a) \right)^2 \right] - 
\\
&~ \frac{1}{n} \sum_{(s,a,r,s') \in \Dcal} \left(g_1(s,a) - r - \gamma f(s',\pi) \right)^2 + \frac{1}{n} \sum_{(s,a,r,s') \in \Dcal} \left(g_2(s,a)  - r - \gamma f(s',\pi) \right)^2 \bigg|
\\
\label{eq:msbobernstein_original}
= &~ \bigg| \E_{\mu} \left[\left(g_1(s,a) - \left(\Tcal^\pi f\right)(s,a) \right)^2\right] - \E_{\mu} \left[\left(g_2(s,a) - \left(\Tcal^\pi f\right)(s,a) \right)^2 \right] - 
\\
&~ \frac{1}{n} \sum_{(s,a,r,s') \in \Dcal} \left( \left(g_1(s,a) - g_2(s,a) \right) \left(g_1(s,a) + g_2(s,a)  - 2 r - 2 \gamma f(s',\pi) \right) \right) \bigg|
\\
\label{eq:msbobernstein}
\leq &~ \sqrt{\frac{4\V_{\mu\times(\Pcal,R)} \left[ \left(g_1(s,a) - g_2(s,a) \right) \left(g_1(s,a) + g_2(s,a)  - 2 r - 2 \gamma f(s',\pi) \right) \right] \log \frac{|\Fcal||\Pi|}{\delta}}{n}}
\\ &~ + \frac{2 \Vmax^2 \log \frac{|\Fcal||\Pi|}{\delta}}{3n},
\end{align}
where the first equation follows from \Eqref{eq:popmsbo_middle} and the last inequality follows from the Bernsiten's inequality.

We now study the term $\V_{\mu\times(\Pcal,R)} \left[ \left(g_1(s,a) - g_2(s,a) \right) \left(g_1(s,a) + g_2(s,a)  - 2 r - 2 \gamma f(s',\pi) \right) \right]$,
\begin{align}
&~ \V_{\mu\times(\Pcal,R)} \left[ \left(g_1(s,a) - g_2(s,a) \right) \left(g_1(s,a) + g_2(s,a)  - 2 r - 2 \gamma f(s',\pi) \right) \right]
\\
\leq &~ \E_{\mu\times(\Pcal,R)} \left[ \left(g_1(s,a) - g_2(s,a) \right)^2 \left(g_1(s,a) + g_2(s,a)  - 2 r - 2 \gamma f(s',\pi) \right)^2 \right]
\\
\label{eq:msbobernsteinvar}
\leq &~ 4 \Vmax^2\E_{\mu} \left[ \left(g_1(s,a) - g_2(s,a) \right)^2 \right]
\end{align}
where the last inequality follows from the fact of $|g_1(s,a) + g_2(s,a)  - 2 r - 2 \gamma f(s',\pi) | \leq 2 \Vmax$.
This completes the proof.
\end{proof}

\begin{lemma}
\label{lem:msbosol}
For any $\pi \in \Pi$, let $f_\pi$ and $g$ be defined as follows,
\begin{align}
f_\pi \coloneqq &~ \argmin_{f \in \Fcal}\sup_{\text{admissible }\nu} \left\|f - \Tcal^\pi f\right\|_{2,\nu}^2
\\
g \coloneqq &~ \argmin_{g \in \Fcal} \frac{1}{n} \sum_{(s,a,r,s') \in \Dcal} \left(g(s,a) - r - \gamma f_\pi(s',\pi) \right)^2.
\end{align}
Then, we have
\begin{align}
\|f_\pi - g\|_{2,\mu} \leq 9 \Vmax \sqrt{\frac{ \log \frac{|\Fcal||\Pi|}{\delta}}{n}} + 2\sqrt{\varepsilon_{\Fcal}}.
\end{align}
\end{lemma}

\begin{proof}[\bf\em Proof of Lemma~\ref{lem:msbosol}]

By applying Lemma~\ref{lem:bernstein_general}, we have
\begin{align}
&~ \bigg| \E_{\mu} \left[\left(f_\pi(s,a) - \left(\Tcal^\pi f_\pi\right)(s,a) \right)^2\right] - \E_{\mu} \left[\left(g(s,a) - \left(\Tcal^\pi f_\pi\right)(s,a) \right)^2 \right] - 
\\
&~ \frac{1}{n} \sum_{(s,a,r,s') \in \Dcal} \left(f_\pi(s,a) - r - \gamma f_\pi(s',\pi) \right)^2 + \frac{1}{n} \sum_{(s,a,r,s') \in \Dcal} \left(g(s,a)  - r - \gamma f_\pi(s',\pi) \right)^2 \bigg|
\\
\leq &~ 4 \Vmax \sqrt{\frac{\E_{\mu} \left[ \left(f_\pi(s,a) - g(s,a) \right)^2 \right] \log \frac{|\Fcal||\Pi|}{\delta}}{n}} + \frac{2 \Vmax^2 \log \frac{|\Fcal||\Pi|}{\delta}}{3n}.
\end{align}

In addition,
\begin{align}
&~ \E_{\mu} \left[ \left(f_\pi(s,a) - g(s,a) \right)^2 \right]
\\
= &~ \|f_\pi - g\|_{2,\mu}^2
\\
\leq &~ 2 \|f_\pi - \Tcal^\pi f_\pi\|_{2,\mu}^2 + 2 \|g - \Tcal^\pi f_\pi\|_{2,\mu}^2
\\
= &~ 2 \|g - \Tcal^\pi f_\pi\|_{2,\mu}^2 - 2 \|f_\pi - \Tcal^\pi f_\pi\|_{2,\mu}^2 + 4 \|f_\pi - \Tcal^\pi f_\pi\|_{2,\mu}^2
\\
\overset{\text{(a)}}{\leq} &~ 2 \|g - \Tcal^\pi f_\pi\|_{2,\mu}^2 - 2 \|f_\pi - \Tcal^\pi f_\pi\|_{2,\mu}^2 + 4 \varepsilon_{\Fcal}
\\
\label{eq:quaeq}
\overset{\text{(b)}}{\leq} &~ 8 \Vmax \sqrt{\frac{\E_{\mu} \left[ \left(f_\pi(s,a) - g(s,a) \right)^2 \right] \log \frac{|\Fcal||\Pi|}{\delta}}{n}} + \frac{4 \Vmax^2 \log \frac{|\Fcal||\Pi|}{\delta}}{3n} + 4 \varepsilon_{\Fcal},
\end{align}
where (a) follows from the fact of $\|f_\pi - \Tcal^\pi f_\pi\|_{2,\mu}^2 \leq \varepsilon_\Fcal$ by Assumption~\ref{asm:relz2}, and (b) is obtained by the following argument
\begin{align}
&~ \|g - \Tcal^\pi f_\pi\|_{2,\mu}^2 - \|f_\pi - \Tcal^\pi f_\pi\|_{2,\mu}^2
\\
\leq &~ \frac{1}{n} \sum_{(s,a,r,s') \in \Dcal} \left(g(s,a)  - r - \gamma f_\pi(s',\pi) \right)^2 - \frac{1}{n} \sum_{(s,a,r,s') \in \Dcal} \left(f_\pi(s,a) - r - \gamma f_\pi(s',\pi) \right)^2
\\
&~ + 4 \Vmax \sqrt{\frac{\E_{\mu} \left[ \left(f_\pi(s,a) - g(s,a) \right)^2 \right] \log \frac{|\Fcal||\Pi|}{\delta}}{n}} + \frac{2 \Vmax^2 \log \frac{|\Fcal||\Pi|}{\delta}}{3n} \tag{by \Eqref{eq:msbobernstein}}
\\
\leq &~ 4 \Vmax \sqrt{\frac{\E_{\mu} \left[ \left(f_\pi(s,a) - g(s,a) \right)^2 \right] \log \frac{|\Fcal||\Pi|}{\delta}}{n}} + \frac{2 \Vmax^2 \log \frac{|\Fcal||\Pi|}{\delta}}{3n} \tag{by the optimality of $g$}
\end{align}

By solving \Eqref{eq:quaeq}, we obtain
\begin{align}
&~ \sqrt{\E_{\mu} \left[ \left(f_\pi(s,a) - g(s,a) \right)^2 \right]}
\\
\leq &~ 4 \Vmax \sqrt{\frac{ \log \frac{|\Fcal||\Pi|}{\delta}}{n}} + 2 \sqrt{\frac{5 \Vmax^2 \log \frac{|\Fcal||\Pi|}{\delta}}{n} + \varepsilon_{\Fcal}}
\\
\leq &~ 4 \Vmax \sqrt{\frac{ \log \frac{|\Fcal||\Pi|}{\delta}}{n}} + 2 \Vmax \sqrt{\frac{5 \log \frac{|\Fcal||\Pi|}{\delta}}{n}} + 2\sqrt{\varepsilon_{\Fcal}}
\\
= &~ 9 \Vmax \sqrt{\frac{ \log \frac{|\Fcal||\Pi|}{\delta}}{n}} + 2\sqrt{\varepsilon_{\Fcal}}.
\end{align}
This completes the proof.
\end{proof}

\subsubsection{Detailed Proofs of Theorem~\ref{thm:version_space} and Theorem~\ref{thm:mspo2be}}
\label{appx:peproofs}

\begin{proof}[\bf\em Proof of Theorem~\ref{thm:version_space}]
Let $g$ be defined as
\begin{align}
g \coloneqq &~ \argmin_{g \in \Fcal} \frac{1}{n} \sum_{(s,a,r,s') \in \Dcal} \left(g(s,a) - r - \gamma f_\pi(s',\pi) \right)^2.
\end{align}
By applying Lemma~\ref{lem:bernstein_general},
\begin{align}
&~ \bigg| \left\|f_\pi - \Tcal^\pi f_\pi  \right\|_{2,\mu}^2 - \left\|g - \Tcal^\pi f_\pi  \right\|_{2,\mu}^2 -
\\
&~ \frac{1}{n} \sum_{(s,a,r,s') \in \Dcal} \left(f_\pi(s,a) - r - \gamma f_\pi(s',\pi) \right)^2 + \frac{1}{n} \sum_{(s,a,r,s') \in \Dcal} \left(g(s,a)  - r - \gamma f_\pi(s',\pi) \right)^2 \bigg|
\\
\leq &~ 4 \Vmax \|f_\pi - g\|_{2,\mu}\sqrt{\frac{ \log \frac{|\Fcal||\Pi|}{\delta}}{n}} + \frac{2 \Vmax^2 \log \frac{|\Fcal||\Pi|}{\delta}}{3n}.
\end{align}

By Lemma~\ref{lem:msbosol}, we know $\|f_\pi - g\|_{2,\mu} \leq 9 \Vmax \sqrt{\frac{ \log \frac{|\Fcal||\Pi|}{\delta}}{n}} + 2\sqrt{\varepsilon_{\Fcal}}$. Thus,
\begin{align}
&~ \bigg| \left\|f_\pi - \Tcal^\pi f_\pi  \right\|_{2,\mu}^2 - \left\|g - \Tcal^\pi f_\pi  \right\|_{2,\mu}^2 -
\\
&~ \frac{1}{n} \sum_{(s,a,r,s') \in \Dcal} \left(f_\pi(s,a) - r - \gamma f_\pi(s',\pi) \right)^2 + \frac{1}{n} \sum_{(s,a,r,s') \in \Dcal} \left(g(s,a)  - r - \gamma f_\pi(s',\pi) \right)^2 \bigg|
\\
\leq &~ 4 \Vmax \left( 9 \Vmax \sqrt{\frac{ \log \frac{|\Fcal||\Pi|}{\delta}}{n}} + 2\sqrt{\varepsilon_{\Fcal}} \right) \sqrt{\frac{ \log \frac{|\Fcal||\Pi|}{\delta}}{n}} + \frac{2 \Vmax^2 \log \frac{|\Fcal||\Pi|}{\delta}}{3n}
\\
\label{eq:concentbd}
= &~ 8 \Vmax \sqrt{\frac{ \log \frac{|\Fcal||\Pi|}{\delta}}{n} \varepsilon_{\Fcal}} + \frac{37 \Vmax^2 \log \frac{|\Fcal||\Pi|}{\delta}}{n}.
\end{align}

We now bound $\|f_\pi - \Tcal^\pi f_\pi  \|_{2,\mu}^2 - \|g - \Tcal^\pi f_\pi  \|_{2,\mu}^2$,
\begin{align}
&~ \left\|f_\pi - \Tcal^\pi f_\pi  \right\|_{2,\mu}^2 - \left\|g - \Tcal^\pi f_\pi  \right\|_{2,\mu}^2
\\
\leq &~ \left( \left\|f_\pi - \Tcal^\pi f_\pi  \right\|_{2,\mu} + \left\|g - \Tcal^\pi f_\pi  \right\|_{2,\mu}\right) \left| \left\|f_\pi - \Tcal^\pi f_\pi  \right\|_{2,\mu} - \left\|g - \Tcal^\pi f_\pi  \right\|_{2,\mu} \right|
\\
\overset{\text{(a)}}{\leq} &~ \left( 2 \left\|f_\pi - \Tcal^\pi f_\pi  \right\|_{2,\mu} + \left\|f_\pi - g \right\|_{2,\mu}\right) \left\|f_\pi - g \right\|_{2,\mu}
\\
\overset{\text{(b)}}{\leq} &~ \left( 4 \sqrt{\varepsilon_{\Fcal}} + 9 \Vmax \sqrt{\frac{ \log \frac{|\Fcal||\Pi|}{\delta}}{n}} \right) \left( 9 \Vmax \sqrt{\frac{ \log \frac{|\Fcal||\Pi|}{\delta}}{n}} + 2\sqrt{\varepsilon_{\Fcal}} \right)
\\
\label{eq:popbound}
= &~ 54 \Vmax \sqrt{\frac{ \log \frac{|\Fcal||\Pi|}{\delta}}{n} \varepsilon_{\Fcal}} + 81 \Vmax^2  \frac{ \log \frac{|\Fcal||\Pi|}{\delta}}{n} + 8 \varepsilon_{\Fcal},
\end{align}
where (a) follows from triangle inequality, and (b) is obtained by Lemma~\ref{lem:msbosol} and the fact of $\|f_\pi - \Tcal^\pi f_\pi\|_{2,\mu}^2 \leq \varepsilon_\Fcal$ (Assumption~\ref{asm:relz2}).

Combining \Eqref{eq:concentbd} and \Eqref{eq:popbound}, we obtain
\begin{align}
&~ \frac{1}{n} \sum_{(s,a,r,s') \in \Dcal} \left(f_\pi(s,a) - r - \gamma f_\pi(s',\pi) \right)^2 - \frac{1}{n} \sum_{(s,a,r,s') \in \Dcal} \left(g(s,a)  - r - \gamma f_\pi(s',\pi) \right)^2
\\
\leq &~ 8 \Vmax \sqrt{\frac{ \log \frac{|\Fcal||\Pi|}{\delta}}{n} \varepsilon_{\Fcal}} + \frac{37 \Vmax^2 \log \frac{|\Fcal||\Pi|}{\delta}}{n} + 54 \Vmax \sqrt{\frac{ \log \frac{|\Fcal||\Pi|}{\delta}}{n} \varepsilon_{\Fcal}} + 81 \Vmax^2 \frac{ \log \frac{|\Fcal||\Pi|}{\delta}}{n} + 8 \varepsilon_{\Fcal}
\\
\leq &~ 62 \Vmax \sqrt{\frac{ \log \frac{|\Fcal||\Pi|}{\delta}}{n} \varepsilon_{\Fcal}} + \frac{118 \Vmax^2 \log \frac{|\Fcal||\Pi|}{\delta}}{n} + 8 \varepsilon_{\Fcal}
\\
\leq &~ \frac{139 \Vmax^2 \log \frac{|\Fcal||\Pi|}{\delta}}{n} + 39 \varepsilon_{\Fcal},
\end{align}
where the last inequality follows from Cauchy–Schwarz inequality. This completes the proof.
\end{proof}

\begin{proof}[\bf\em Proof of Theorem~\ref{thm:mspo2be}]
Let $g$ be defined as follows,
\begin{align}
g \coloneqq \argmin_{f' \in \Fcal}\frac{1}{n} \sum_{(s,a,r,s') \in \Dcal} \left(f'(s,a)  - r - \gamma f(s',\pi) \right)^2.
\end{align}

\paragraph{Bounding $\|g - \Tcal^\pi f\|_{2,\mu}$.}
We show that $g$ could approximate $\Tcal^\pi f$ well over distribution $\mu$. We define $f_1$ as follows
\begin{align}
f_1 \coloneqq \argmin_{f' \in \Fcal} \left\| f' - \Tcal^\pi f\right\|_{2,\mu}^2.
\end{align}

Then, we consider the following function
\begin{align}
\frac{1}{n} \sum_{(s,a,r,s') \in \Dcal} \left(g(s,a)  - r - \gamma f(s',\pi) \right)^2 - \frac{1}{n} \sum_{(s,a,r,s') \in \Dcal} \left(f_1(s,a)  - r - \gamma f(s',\pi) \right)^2.
\end{align}
We have
\begin{align}
&~ \E_{\mu\times(\Pcal,R)} \left[ \frac{1}{n} \sum_{(s,a,r,s') \in \Dcal} \left(g(s,a)  - r - \gamma f(s',\pi) \right)^2 - \frac{1}{n} \sum_{(s,a,r,s') \in \Dcal} \left(f_1(s,a)  - r - \gamma f(s',\pi) \right)^2 \right]
\\
= &~ \left\| g - \Tcal^\pi f\right\|_{2,\mu}^2 - \left\| f_1 - \Tcal^\pi f\right\|_{2,\mu}^2,
\end{align}
by similar arguments of \Eqref{eq:empmsbo} and \Eqref{eq:popmsbo}.

Then
\begin{align}
&~ \bigg| \left\| g - \Tcal^\pi f\right\|_{2,\mu}^2 - \left\| f_1 - \Tcal^\pi f\right\|_{2,\mu}^2 - \frac{1}{n} \sum_{(s,a,r,s') \in \Dcal} \left(g(s,a)  - r - \gamma f(s',\pi) \right)^2
\\
&~ + \frac{1}{n} \sum_{(s,a,r,s') \in \Dcal} \left(f_1(s,a)  - r - \gamma f(s',\pi) \right)^2\bigg|
\\
\leq &~ 4 \Vmax \|g - f_1\|_{2,\mu}\sqrt{\frac{ \log \frac{|\Fcal||\Pi|}{\delta}}{n}} + \frac{2 \Vmax^2 \log \frac{|\Fcal||\Pi|}{\delta}}{3n},
\end{align}
where the inequality follows from Lemma~\ref{lem:bernstein_general}.

Thus,
\begin{align}
&~ \left\| g - \Tcal^\pi f\right\|_{2,\mu}^2
\\
\leq &~ \frac{1}{n} \sum_{(s,a,r,s') \in \Dcal} \left(g(s,a)  - r - \gamma f(s',\pi) \right)^2 - \frac{1}{n} \sum_{(s,a,r,s') \in \Dcal} \left(f_1(s,a)  - r - \gamma f(s',\pi) \right)^2
\\
&~ + \left\| f_1 - \Tcal^\pi f\right\|_{2,\mu}^2 + 4 \Vmax \|g - f_1\|_{2,\mu}\sqrt{\frac{ \log \frac{|\Fcal||\Pi|}{\delta}}{n}} + \frac{2 \Vmax^2 \log \frac{|\Fcal||\Pi|}{\delta}}{3n}
\\
\leq &~ \left\| f_1 - \Tcal^\pi f\right\|_{2,\mu}^2 + 4 \Vmax \|g - f_1\|_{2,\mu}\sqrt{\frac{ \log \frac{|\Fcal||\Pi|}{\delta}}{n}} + \frac{2 \Vmax^2 \log \frac{|\Fcal||\Pi|}{\delta}}{3n}
\\
\leq &~ \varepsilon_{\Fcal,\Fcal} + 4 \Vmax \|g - \Tcal^\pi f\|_{2,\mu}\sqrt{\frac{ \log \frac{|\Fcal||\Pi|}{\delta}}{n}} + 4 \Vmax \sqrt{\varepsilon_{\Fcal,\Fcal}}\sqrt{\frac{ \log \frac{|\Fcal||\Pi|}{\delta}}{n}} + \frac{2 \Vmax^2 \log \frac{|\Fcal||\Pi|}{\delta}}{3n}
\\
\label{eq:quaform}
\leq &~ 4 \Vmax \|g - \Tcal^\pi f\|_{2,\mu}\sqrt{\frac{ \log \frac{|\Fcal||\Pi|}{\delta}}{n}} + \frac{5 \Vmax^2 \log \frac{|\Fcal||\Pi|}{\delta}}{n} + 5\varepsilon_{\Fcal,\Fcal},
\end{align}
where the second inequality follows from the optimality of $g$, and the third inequality is by Assumption~\ref{asm:comp_restat}, and the last equation follows from the Cauchy–Schwarz inequality.

By solving \Eqref{eq:quaform}, we obtain
\begin{align}
&~ \|g - \Tcal^\pi f\|_{2,\mu}
\\
\leq &~ 2 \Vmax \sqrt{\frac{ \log \frac{|\Fcal||\Pi|}{\delta}}{n}} + \sqrt{\frac{5 \Vmax^2 \log \frac{|\Fcal||\Pi|}{\delta}}{n} + 5\varepsilon_{\Fcal,\Fcal}}
\\
\leq &~ 2 \Vmax \sqrt{\frac{ \log \frac{|\Fcal||\Pi|}{\delta}}{n}} + \Vmax \sqrt{\frac{5 \log \frac{|\Fcal||\Pi|}{\delta}}{n}} + \sqrt{5\varepsilon_{\Fcal,\Fcal}}
\\
\label{eq:gbound}
= &~ 5 \Vmax \sqrt{\frac{ \log \frac{|\Fcal||\Pi|}{\delta}}{n}} + \sqrt{5\varepsilon_{\Fcal,\Fcal}}.
\end{align}

\paragraph{Bounding $\|f - \Tcal^\pi f\|_{2,\mu}$.}
Similar to \Eqref{eq:empmsbo} and \Eqref{eq:popmsbo}, we have
\begin{align}
&~ \frac{1}{n} \sum_{(s,a,r,s') \in \Dcal} \left(f(s,a) - r - \gamma f(s',\pi) \right)^2 - \frac{1}{n} \sum_{(s,a,r,s') \in \Dcal} \left(g(s,a)  - r - \gamma f(s',\pi) \right)^2
\\
= &~ \frac{1}{n} \sum_{(s,a,r,s') \in \Dcal} \left( \left(f(s,a) - r - \gamma f(s',\pi) \right)^2 - \left(g(s,a)  - r - \gamma f(s',\pi) \right)^2 \right)
\\
\label{eq:empmsbo2}
= &~ \frac{1}{n} \sum_{(s,a,r,s') \in \Dcal} \left( \left(f(s,a) - g(s,a) \right) \left(f(s,a) + g(s,a)  - 2 r - 2 \gamma f(s',\pi) \right) \right)
\end{align}
and
\begin{align}
&~ \E_{\mu\times(\Pcal,R)} \left[\left(f(s,a) - r - \gamma f(s',\pi) \right)^2\right] - \E_{\mu\times(\Pcal,R)} \left[\left(g(s,a) - r - \gamma f(s',\pi) \right)^2 \right]
\\
= &~ \E_{\mu\times(\Pcal,R)} \left[ \left(f(s,a) - g(s,a) \right) \left(f(s,a) + g(s,a)  - 2 r - 2 \gamma f(s',\pi) \right) \right]
\\
= &~ \E_\mu \left[ \E \left[ \left(f(s,a) - g(s,a) \right) \left(f(s,a) + g(s,a)  - 2 r - 2 \gamma f(s',\pi) \right) \middle| s,a \right] \right]
\\
\label{eq:popmsbo_middle2}
= &~ \E_\mu \left[  \left(f(s,a) - g(s,a) \right) \left(f(s,a) + g_2(s,a)  - 2 \left(\Tcal^\pi f\right)(s,a) \right) \right]
\\
\label{eq:popmsbo2}
= &~ \E_{\mu} \left[\left(f(s,a) - \left(\Tcal^\pi f\right)(s,a) \right)^2\right] - \E_{\mu} \left[\left(g(s,a) - \left(\Tcal^\pi f\right)(s,a) \right)^2 \right].
\end{align}
It implies that
\begin{align}
&~ \| f - \Tcal^\pi f\|_{2,\mu}^2 - \left\| g - \Tcal^\pi f\right\|_{2,\mu}^2
\\
= &~ \E_{\mu} \left[\left(f(s,a) - \left(\Tcal^\pi f\right)(s,a) \right)^2\right] - \E_{\mu} \left[\left(g(s,a) - \left(\Tcal^\pi f\right)(s,a) \right)^2 \right]
\\
= &~ \E_{\mu\times(\Pcal,R)} \left[ \frac{1}{n} \sum_{(s,a,r,s') \in \Dcal} \left(f(s,a) - r - \gamma f(s',\pi) \right)^2 - \frac{1}{n} \sum_{(s,a,r,s') \in \Dcal} \left(g(s,a)  - r - \gamma f(s',\pi) \right)^2 \right].
\end{align}

By applying Lemma~\ref{lem:bernstein_general},
\begin{align}
&~ \bigg| \| f - \Tcal^\pi f\|_{2,\mu}^2 - \left\| g - \Tcal^\pi f\right\|_{2,\mu}^2
\\
&~ - \frac{1}{n} \sum_{(s,a,r,s') \in \Dcal} \left(f(s,a) - r - \gamma f(s',\pi) \right)^2 + \frac{1}{n} \sum_{(s,a,r,s') \in \Dcal} \left(g(s,a)  - r - \gamma f(s',\pi) \right)^2 \bigg|
\\
\leq &~ 4 \Vmax \|f - g\|_{2,\mu}\sqrt{\frac{ \log \frac{|\Fcal||\Pi|}{\delta}}{n}} + \frac{2 \Vmax^2 \log \frac{|\Fcal||\Pi|}{\delta}}{3n}
\\
\leq &~ 4 \Vmax \left( \|f - \Tcal^\pi f\|_{2,\mu} + \|g - \Tcal^\pi f\|_{2,\mu} \right)\sqrt{\frac{ \log \frac{|\Fcal||\Pi|}{\delta}}{n}} + \frac{2 \Vmax^2 \log \frac{|\Fcal||\Pi|}{\delta}}{3n}
\\
\leq &~ 4 \Vmax \|f - \Tcal^\pi f\|_{2,\mu} \sqrt{\frac{ \log \frac{|\Fcal||\Pi|}{\delta}}{n}} + 4 \Vmax \sqrt{\frac{ \log \frac{|\Fcal||\Pi|}{\delta}}{n} \varepsilon_{\Fcal,\Fcal}} + \frac{13 \Vmax^2 \log \frac{|\Fcal||\Pi|}{\delta}}{n},
\end{align}
where the last inequality follows from \Eqref{eq:gbound}.

Rearranging the inequality above, we obtain
\begin{align}
&~ \| f - \Tcal^\pi f\|_{2,\mu}^2
\\
\leq &~ \left\| g - \Tcal^\pi f\right\|_{2,\mu}^2 + \frac{1}{n} \sum_{(s,a,r,s') \in \Dcal} \left(f(s,a) - r - \gamma f(s',\pi) \right)^2 - \frac{1}{n} \sum_{(s,a,r,s') \in \Dcal} \left(g(s,a)  - r - \gamma f(s',\pi) \right)^2
\\
&~ + 4 \Vmax \|f - \Tcal^\pi f\|_{2,\mu} \sqrt{\frac{ \log \frac{|\Fcal||\Pi|}{\delta}}{n}} + 4 \Vmax \sqrt{\frac{ \log \frac{|\Fcal||\Pi|}{\delta}}{n} \varepsilon_{\Fcal,\Fcal}} + \frac{13 \Vmax^2 \log \frac{|\Fcal||\Pi|}{\delta}}{n}
\\
\overset{\text{(a)}}{\leq} &~ \left( 5 \Vmax \sqrt{\frac{ \log \frac{|\Fcal||\Pi|}{\delta}}{n}} + \sqrt{\varepsilon_{\Fcal,\Fcal}} \right)^2 + \varepsilon
\\
&~ + 4 \Vmax \|f - \Tcal^\pi f\|_{2,\mu} \sqrt{\frac{ \log \frac{|\Fcal||\Pi|}{\delta}}{n}} + 4 \Vmax \sqrt{\frac{ \log \frac{|\Fcal||\Pi|}{\delta}}{n} \varepsilon_{\Fcal,\Fcal}} + \frac{13 \Vmax^2 \log \frac{|\Fcal||\Pi|}{\delta}}{n}
\\
\label{eq:fbequaform}
= &~ 4 \Vmax \|f - \Tcal^\pi f\|_{2,\mu} \sqrt{\frac{ \log \frac{|\Fcal||\Pi|}{\delta}}{n}} + 14 \Vmax \sqrt{\frac{ \log \frac{|\Fcal||\Pi|}{\delta}}{n} \varepsilon_{\Fcal,\Fcal}} + \frac{38 \Vmax^2 \log \frac{|\Fcal||\Pi|}{\delta}}{n} + \varepsilon_{\Fcal,\Fcal} + \varepsilon.
\end{align}
where (a) follows from \Eqref{eq:gbound} and the definition of $\varepsilon$ in the theorem statement.

Solving the quadratic form of \Eqref{eq:fbequaform}, we have
\begin{align}
&~ \|f - \Tcal^\pi f\|_{2,\mu}
\\
\leq &~ 2 \Vmax \sqrt{\frac{ \log \frac{|\Fcal||\Pi|}{\delta}}{n}} + \sqrt{\frac{ 38 \Vmax^2 \log \frac{|\Fcal||\Pi|}{\delta}}{n} + 14 \Vmax \sqrt{\frac{ \log \frac{|\Fcal||\Pi|}{\delta}}{n} \varepsilon_{\Fcal,\Fcal}} + \varepsilon_{\Fcal,\Fcal} + \varepsilon}
\\
\leq &~ 2 \Vmax \sqrt{\frac{ \log \frac{|\Fcal||\Pi|}{\delta}}{n}} + \sqrt{\frac{ 38 \Vmax^2 \log \frac{|\Fcal||\Pi|}{\delta}}{n}} + \sqrt{14 \Vmax \sqrt{\frac{ \log \frac{|\Fcal||\Pi|}{\delta}}{n} \varepsilon_{\Fcal,\Fcal}}} + \sqrt{\varepsilon_{\Fcal,\Fcal} + \varepsilon}
\\
\leq &~ \Vmax \sqrt{\frac{ 67 \log \frac{|\Fcal||\Pi|}{\delta}}{n}} + \sqrt[4]{\frac{196 \Vmax^2 \log \frac{|\Fcal||\Pi|}{\delta}}{n} \varepsilon_{\Fcal,\Fcal}} + \sqrt{\varepsilon_{\Fcal,\Fcal} + \varepsilon}
\\
\leq &~ \Vmax \sqrt{\frac{ 231 \log \frac{|\Fcal||\Pi|}{\delta}}{n}} + \sqrt{\varepsilon_{\Fcal,\Fcal}} + \sqrt{\varepsilon_{\Fcal,\Fcal} + \varepsilon},
\end{align}
where the last inequality follows from Cauchy–Schwarz inequality. This completes the proof.
\end{proof}

\subsection{Results for Linear Function Approximation}
\label{sec:PEresults_linearMDP}

We now provide concentration analysis for linear function approximation.

The results for linear function approximation differ from those in general function approximations in two perspectives: (1) Since our linear function approximation setup are well specified, we have $\varepsilon_{\Fcal_\Phi} = \varepsilon_{\Fcal_\Phi,\Fcal_\Phi} = 0$. It also implies that $Q^\pi \in \Fcal$, $\forall \pi \in \Pi_\Phi$. (2) The uniform convergence argument for $\Fcal_\Phi$ and $\Pi_\Phi$ can be studied more precisely (Lemma~\ref{lem:bernstein_linear} in Appendix~\ref{appx:linearmdppeproof}). We summarize the results in this section, and we defer the detailed proof to Appendix~\ref{appx:linearmdppeproof}. 

\begin{corollary}[Alternative of Theorem~\ref{thm:version_space} in Linear Function Approximation]
\label{cor:version_space_linear}
For any $\pi \in \Pi_\Phi$, we have
\begin{align}
\label{eq:deferlinearmdp}
\Ecal(Q^\pi,\pi;\Dcal) \leq \frac{c \Vmax^2 d \log \frac{n L_1 L_2}{\Vmax \delta}}{n},
\end{align}
where $c$ is an absolute constant.
\end{corollary}
We define the RHS of \Eqref{eq:deferlinearmdp} to be $\varepsilon_r$ in linear function approximation.

\begin{corollary}[Alternative of Theorem~\ref{thm:mspo2be} in Linear Function Approximation]
\label{cor:mspo2be_linear}
For any $\pi \in \Pi_\Phi$ and $f \in \Fcal_\Phi$, if $\Ecal(f,\pi;\Dcal) \leq \varepsilon$ for any $\varepsilon > 0$, then,
\begin{align}
\label{eq:linearmdpbe}
\| f - \Tcal^\pi f\|_{2,\mu} \leq c \Vmax \sqrt{\frac{d \log \frac{n L_1 L_2}{\Vmax \delta}}{n}} + \sqrt{\varepsilon},
\end{align}
where $c$ is an absolute constant.
\end{corollary}

\begin{theorem}
\label{thm:mspo2be_linear_emp}
For any $\pi \in \Pi_\Phi$ and $f \in \Fcal_\Phi$, if $\Ecal(f,\pi;\Dcal) \leq \varepsilon$ for any $\varepsilon > 0$, then,
\begin{align}
\label{eq:linearmdpbe1}
\| f - \Tcal^\pi f\|_{2,\Dcal} \leq c \Vmax \sqrt{\frac{d \log \frac{n L_1 L_2}{\Vmax \delta}}{n}} + \sqrt{\varepsilon},
\end{align}
where $c$ is an absolute constant.
\end{theorem}
We also define the RHS of \Eqref{eq:linearmdpbe1} with $\varepsilon = \varepsilon_r$ to be $\sqrt{\varepsilon_b}$ in linear function approximation.

\subsubsection{Detailed Proofs for Linear Function Approximation Results}
\label{appx:linearmdppeproof}

This section provides concentration analysis for linear function approximation\footnote{The concentration analysis of the linear algorithm in this version is different from the one in the NeurIPS version. The NeurIPS version selected $\Pi_\Phi$ to be the class of all greedy policies w.r.t.~$\Fcal_\Phi$ and used $L_1$ covering number to characterize the complexities, and the analysis might have missed an additional $|\Acal|$ factor when applying the original Lemma A.11 to prove the original Lemma A.12. This version fixed that issue by using $L_\infty$ covering number and selecting the policy class $\Pi_\Phi$ to be the softmax class.} (Definition~\ref{def:linearmdp}) using covering number with the following metrics for the value-function class and the policy class, respectively:
\begin{align}
\label{eq:def_infnorm}
\|f_1 - f_2\|_\infty = &~ \sup_{(s,a) \in \Scal \times \Acal} |f_1(s,a) - f_2(s,a)|,
\\
\|\pi_1 - \pi_2\|_{\infty,1} = &~ \sup_{s \in \Scal} \|\pi_1(\cdot|s) - \pi_2(\cdot|s)\|_1.
\end{align}
The $\varepsilon$-covering number with metric $\rho$ is defined as follows.
\begin{definition}[$\varepsilon$-covering number]
We define the $\varepsilon$-covering number of a set $\Fcal$ under metric $\rho$ to be the the cardinality of the smallest $\varepsilon$-coverm, denoted by $\Ncal_\rho(\Fcal,\varepsilon)$. The $\varepsilon$-cover of a set $\Fcal$ w.r.t.~a metric $\rho$ is a set $\{g_1, \dotsc, g_n\} \subseteq \Fcal$, such that for each $g \in \Fcal$, there exists some $g_i \in \{g_1, \dotsc, g_n\}$ such that $\rho(g,g_i) \leq \varepsilon$.
\end{definition}

\begin{lemma}[Lemma~\ref{lem:bernstein_general} with $\varepsilon$-covering number, Lemma 10 of~\citep{cheng2022adversarially}]
\label{lem:bernstein_general_L_infty}
With probability at least $1-\delta$, for any $f, g_1, g_2 \in \Fcal$ and $\pi \in \Pi$,
\begin{align}
&~ \bigg| \left\| g_1 - \Tcal^\pi f\right\|_{2,\mu}^2 - \left\| g_2 - \Tcal^\pi f\right\|_{2,\mu}^2
\\
&~ - \frac{1}{N} \sum_{(s,a,r,s') \in \Dcal} \left(g_1(s,a) - r - \gamma f(s',\pi) \right)^2 + \frac{1}{N} \sum_{(s,a,r,s') \in \Dcal} \left(g_2(s,a)  - r - \gamma f(s',\pi) \right)^2 \bigg|
\\
\leq &~ \Ocal \left( \Vmax \|g_1 - g_2\|_{2,\mu}\sqrt{\frac{ \log \frac{\Ncal_\infty(\Fcal,\frac{\Vmax}{n})\Ncal_{\infty,1}(\Pi,\frac{1}{n})}{\delta}}{n}} + \frac{\Vmax^2 \log \frac{\Ncal_\infty(\Fcal,\frac{\Vmax}{n})\Ncal_{\infty,1}(\Pi,\frac{1}{n})}{\delta}}{n} \right).
\end{align}
\end{lemma}
\begin{lemma}[$\varepsilon$-covering Number for $\Fcal_\Phi$ and $\Pi_\Phi$, Lemma 6 of~\citep{zanette2021provable}]
\label{lem:linear_covering_num}
Let $\Fcal_\Phi$ and $\Pi_\Phi$ be defined in Definition~\ref{def:linearmdp}. Then, for any $\varepsilon \in [0,1]$, we have
\begin{align}
\log(\Ncal_\infty(\Fcal_\Phi,\varepsilon)) \leq &~ d \log(1 + \nicefrac{2 L_1}{\varepsilon}),
\\
\log(\Ncal_{\infty,1}(\Pi_\Phi,\varepsilon)) \leq &~ d \log(1 + \nicefrac{16 L_2}{\varepsilon}).
\end{align}
\end{lemma}
Applying Lemma~\ref{lem:linear_covering_num} to Lemma~\ref{lem:bernstein_general_L_infty} directly implies the following lemma.
\begin{lemma}[Lemma~\ref{lem:bernstein_general_L_infty} in Linear Function Approximation]
\label{lem:bernstein_linear}
For any $f, g_1, g_2 \in \Fcal_\Phi$ and $\pi \in \Pi_\Phi$, we have
\begin{align}
&~ \bigg| \| g_1 - \Tcal^\pi f\|_{2,\mu}^2 - \left\| g_2 - \Tcal^\pi f\right\|_{2,\mu}^2
\\
&~ - \frac{1}{n} \sum_{(s,a,r,s') \in \Dcal} \left(g_1(s,a) - r - \gamma f(s',\pi) \right)^2 + \frac{1}{n} \sum_{(s,a,r,s') \in \Dcal} \left(g_2(s,a)  - r - \gamma f(s',\pi) \right)^2 \bigg|
\\
\leq &~ c \Vmax \left\| g_1(s,a) - g_2(s,a) \right\|_{2,\mu} \sqrt{\frac{ d \log \frac{n L_1 L_2}{\Vmax \delta}}{n}} + c' \frac{\Vmax^2 d \log \frac{n L_1 L_2}{\Vmax \delta}}{n},
\end{align}
where $c$ and $c'$ are absolute constants.
\end{lemma}

\begin{corollary}[Alternative of Lemma~\ref{lem:fastrate} in Linear Function Approximation]
\label{cor:fastrate_linear}
For any $f_1, f_2 \in \Fcal_\Phi$ and $\pi \in \Pi_\Phi$, w.p.~$1 - \delta$,
\begin{align}
\left| \left\|f_1(s,a) - (\Tcal^\pi f_2)(s,a)\right\|_{2,\mu} - \left\|f_1(s,a) - (\Tcal^\pi f_2)(s,a)\right\|_{2,\Dcal} \right| \leq c \Vmax \sqrt{\frac{d \log \frac{n L_1 L_2}{\Vmax \delta}}{n}},
\end{align}
where $c$ is an absolute constant.
\end{corollary}

\begin{proof}[\bf\em Proof of Theorem~\ref{thm:mspo2be_linear_emp}]
By combining Corollary~\ref{cor:fastrate_linear} and Corollary~\ref{cor:mspo2be_linear}, we complete the proof, similar to the proof of Theorem~\ref{thm:mspo2be}.
\end{proof}
\section{Detailed Proofs in Section~\ref{sec:constrdpi}}

\subsection{Detailed Proofs for General Function Approximation}
\label{sec:info_theo_proof}

Over this section, the definition of $\varepsilon_r$ follows from \Eqref{eq:upbdespr}.

\begin{lemma}
For any $\pi \in \Pi$, $\displaystyle \min_{f \in \Fcal_{\pi, \varepsilon_r}} \|Q^\pi - f\|_{2,\nu} \leq \frac{\sqrt{\varepsilon_\Fcal}}{1 - \gamma}$ for any admissible distribution $\nu$.
\end{lemma}
\begin{proof}
Let $f_\pi \coloneqq \argmin_{f \in \Fcal} \max_{\text{admissible }\nu} \|f - \Tcal^\pi f\|_{2,\nu}$. By definition of $\Fcal_{\pi, \varepsilon_r}$, we know $f_\pi \in \Fcal_{\pi, \varepsilon_r}$.
Then,
\begin{align}
\min_{f \in \Fcal_{\pi, \varepsilon_r}} \|Q^\pi - f\|_{2,\nu} \leq \|Q^\pi - f_\pi\|_{2,\nu} \leq \frac{1}{1 - \gamma} \max_{\text{admissible }\nu} \|f_\pi - \Tcal^{\pi} f_\pi\|_{2,\nu} \leq \frac{\sqrt{\varepsilon_\Fcal}}{1 - \gamma}.\tag*{\qedhere}
\end{align}
\end{proof}

\begin{lemma}
\label{lem:vslbd}
For any $\pi \in \Pi$, $\displaystyle \min_{f \in \Fcal_{\pi, \varepsilon_r}} f(s_0,\pi) \leq J(\pi) + \frac{\sqrt{\varepsilon_\Fcal}}{1 - \gamma}$.
\end{lemma}

\begin{proof}[\bf\em Proof of Lemma~\ref{lem:vslbd}]
Let $f_\pi \coloneqq \argmin_{f \in \Fcal} \max_{\text{admissible }\nu} \|f - \Tcal^\pi f\|_{2,\nu}$.
\begin{align}
\min_{f \in \Fcal_{\pi, \varepsilon_r}} f(s_0,\pi) \leq f_\pi(s_0,\pi) \leq Q^\pi(s_0,\pi) + \frac{\sqrt{\varepsilon_\Fcal}}{1 - \gamma} = J(\pi) + \frac{\sqrt{\varepsilon_\Fcal}}{1 - \gamma}. \tag*{\qedhere}
\end{align}
\end{proof}

Therefore, the optimization objective is actually a valid lower bound of $J(\pi)$. Similarly, we have the following symmetrical result.

\begin{lemma}
\label{lem:vsubd}
For any $\pi \in \Pi$, $\displaystyle \max_{f \in \Fcal_{\pi, \varepsilon_r}} f(s_0,\pi) \geq J(\pi) - \frac{\sqrt{\varepsilon_\Fcal}}{1 - \gamma}$.
\end{lemma}
\begin{proof}[\bf\em Proof of Lemma~\ref{lem:vsubd}]
Let $f_\pi \coloneqq \argmin_{f \in \Fcal} \max_{\text{admissible }\nu} \|f - \Tcal^\pi f\|_{2,\nu}$.
\begin{align}
\max_{f \in \Fcal_{\pi, \varepsilon_r}} f(s_0,\pi) \geq f_\pi(s_0,\pi) \geq Q^\pi(s_0,\pi) - \frac{\sqrt{\varepsilon_\Fcal}}{1 - \gamma} = J(\pi) - \frac{\sqrt{\varepsilon_\Fcal}}{1 - \gamma}. \tag*{\qedhere}
\end{align}
\end{proof}

We now ready to provide the proof of Theorem~\ref{thm:infothebd2}.

\begin{proof}[\bf\em Proof of Theorem~\ref{thm:infothebd2}]
Using the optimality of $\pihat$, we have
\begin{align}
\max_{f \in \Fcal_{\pi,\varepsilon_r}}f(s_0,\pi) - \min_{f \in \Fcal_{\pihat,\varepsilon_r}}f(s_0,\pihat) \leq \max_{f \in \Fcal_{\pi,\varepsilon_r}}f(s_0,\pi) - \min_{f \in \Fcal_{\pi,\varepsilon_r}}f(s_0,\pi).
\end{align}
Now, let $f_{\pi,\min} \coloneqq \argmin_{f \in \Fcal_{\pi,\varepsilon_r}}f(s_0,\pi)$ and $f_{\pi,\max} \coloneqq \argmax_{f \in \Fcal_{\pi,\varepsilon_r}}f(s_0,\pi)$. By a standard telescoping argument (e.g., \citep[Lemma 1]{xie2020q}), we can obtain
\begin{align}
&~ f_{\pi,\max}(\pi,s_0) - f_{\pi,\min}(\pi,s_0)
\\
= &~ f_{\pi,\max}(\pi,s_0) - J(\pi) + J(\pi) - f_{\pi,\min}(\pi,s_0)
\\
= &~ \frac{1}{1 - \gamma} \left( \Ebb_{d_\pi} \left[f_{\pi,\max} - \Tcal^\pi f_{\pi,\max} \right] - \Ebb_{d_\pi} \left[f_{\pi,\min} - \Tcal^\pi f_{\pi,\min} \right] \right)
\\
= &~ \frac{1}{1 - \gamma} \big( \Ebb_{\mu} \left[ \nicefrac{\nu}{\mu} \cdot \left(\left(f_{\pi,\max} - \Tcal^\pi f_{\pi,\max} \right) - \left(f_{\pi,\min} - \Tcal^\pi f_{\pi,\min} \right)\right)\right]\\
&~ + \Ebb_{d_\pi} \left[(f_{\pi,\max} - \Tcal^\pi f_{\pi,\max}) - (f_{\pi,\min} - \Tcal^\pi f_{\pi,\min}) \right]\\
&~ - \Ebb_{\nu} \left[(f_{\pi,\max} - \Tcal^\pi f_{\pi,\max}) - (f_{\pi,\min} - \Tcal^\pi f_{\pi,\min}) \right] \big)
\\
= &~ \underbrace{\frac{1}{1 - \gamma} \left( \Ebb_{\mu} \left[ \nicefrac{\nu}{\mu} \cdot \left(\left(f_{\pi,\max} - \Tcal^\pi f_{\pi,\max} \right) - \left(f_{\pi,\min} - \Tcal^\pi f_{\pi,\min} \right)\right)\right]\right)}_{\text{(I)}} \\
&~ + \underbrace{\frac{1}{1 - \gamma} \left(\Ebb_{d_\pi} \left[\Delta f_{\pi} - \gamma\Pcal^\pi \Delta f_{\pi} \right]  - \Ebb_{\nu} \left[\Delta f_{\pi} - \gamma\Pcal^\pi \Delta f_{\pi} \right] \right)}_{\text{(II)}}, \tag{$\Delta f_{\pi} \coloneqq f_{\pi,\max} - f_{\pi,\min}$}
\end{align}
where $\nu$ is an arbitrary on-support state-action distribution. We now discuss these two terms above separately.

For the term (I),
\begin{align}
\text{(I)} \leq &~ \frac{1}{1 - \gamma} \left| \Ebb_{\mu} \left[ \nicefrac{\nu}{\mu} \cdot \left(f_{\pi,\max} - \Tcal^\pi f_{\pi,\max} \right)\right] \right| + \frac{1}{1 - \gamma} \left| \Ebb_{\mu} \left[ \nicefrac{\nu}{\mu} \cdot \left(f_{\pi,\min} - \Tcal^\pi f_{\pi,\min} \right)\right] \right|
\\
\leq &~ \frac{\sqrt{\Cscr(\nu;\mu,\Fcal,\pi)}}{1 - \gamma} \left( \|f_{\pi,\max} - \Tcal^\pi f_{\pi,\max}\|_{2,\mu} + \|f_{\pi,\min} - \Tcal^\pi f_{\pi,\min}\|_{2,\mu}\right),
\end{align}
because of the Cauchy-Schwarz inequality for random variables ($|\Ebb[XY]| \leq \sqrt{\Ebb[X^2]\Ebb[Y^2]}$).

For the term (II),
\begin{align}
\text{(II)} = &~ \frac{1}{1 - \gamma} \left(\Ebb_{d_\pi} \left[\Delta f_{\pi} - \gamma\Pcal^\pi \Delta f_{\pi} \right]  - \Ebb_{\nu} \left[\Delta f_{\pi} - \gamma\Pcal^\pi \Delta f_{\pi} \right] \right)
\\
= &~ \frac{1}{1 - \gamma} \sum_{(s,a) \in \Scal \times \Acal} \left[d_\pi(s,a) - \nu(s,a)\right] \left[\Delta f_{\pi}(s,a) - \gamma(\Pcal^\pi\Delta f_{\pi})(s,a)\right] 
\\
\leq &~ \frac{1}{1 - \gamma} \sum_{(s,a) \in \Scal \times \Acal} \1(d_\pi(s,a) \geq \nu(s,a)) \left[d_\pi(s,a) - \nu(s,a)\right] \left[\Delta f_{\pi}(s,a) - \gamma(\Pcal^\pi\Delta f_{\pi})(s,a)\right] 
\\
&~ + \frac{1}{1 - \gamma} \left| \sum_{(s,a) \in \Scal \times \Acal} \1(\nu(s,a) > d_\pi(s,a)) \left[\nu(s,a) - d_\pi(s,a)\right] \left[\Delta f_{\pi}(s,a) - \gamma(\Pcal^\pi\Delta f_{\pi})(s,a)\right] \right|
\\
\leq &~ \frac{1}{1 - \gamma} \sum_{(s,a) \in \Scal\times\Acal} (d_{\pi}\setminus\nu)(s,a) \left[\Delta f_{\pi}(s,a) - \gamma(\Pcal^\pi\Delta f_{\pi})(s,a)\right]
\\
&~ + \frac{1}{1 - \gamma} \sum_{(s,a) \in \Scal \times \Acal} \1(\nu(s,a) > d_\pi(s,a)) \left[\nu(s,a) - d_\pi(s,a)\right] \left|\Delta f_{\pi}(s,a) - \gamma(\Pcal^\pi\Delta f_{\pi})(s,a)\right|
\\
\leq &~ \frac{1}{1 - \gamma} \sum_{(s,a) \in \Scal\times\Acal} (d_{\pi}\setminus\nu)(s,a) \left[\Delta f_{\pi}(s,a) - \gamma(\Pcal^\pi\Delta f_{\pi})(s,a)\right] \\
&~ + \frac{1}{1 - \gamma} \Ebb_{\nu} \left[ \left|f_{\pi,\max} - \Tcal^\pi f_{\pi,\max} \right| + \left|f_{\pi,\min} - \Tcal^\pi f_{\pi,\min} \right|\right]
\\
\leq &~ \frac{1}{1 - \gamma} \sum_{(s,a) \in \Scal\times\Acal} (d_{\pi}\setminus\nu)(s,a) \left[\Delta f_{\pi}(s,a) - \gamma(\Pcal^\pi\Delta f_{\pi})(s,a)\right] \\
&~ + \frac{\sqrt{\Cscr(\nu;\mu,\Fcal,\pi)}}{1 - \gamma} \left(\|f_{\pi,\max} - \Tcal^\pi f_{\pi,\max}\|_{2,\mu} + \|f_{\pi,\min} - \Tcal^\pi f_{\pi,\min}\|_{2,\mu}\right),
\end{align}
where the second last inequality follows from the fact of $\nu(s,a) \geq \1(\nu(s,a) > d_\pi(s,a))[\nu(s,a) - d_\pi(s,a)]$ for any $(s,a) \in \Scal \times \Acal$ and the triangle inequality for the absolute value, and the last inequality uses the Cauchy-Schwarz inequality for random variables (similar to the argument for the term (I)).

Combining the bounds of both term (I) and term (II), we have
\begin{align}
&~ f_{\pi,\max}(\pi,s_0) - f_{\pi,\min}(\pi,s_0) \\
\label{eq:fminbd_2}
\leq &~ \frac{1}{1 - \gamma} \sum_{(s,a) \in \Scal\times\Acal} (d_{\pi}\setminus\nu)(s,a) \left[\Delta f_{\pi}(s,a) - \gamma(\Pcal^\pi\Delta f_{\pi})(s,a)\right] \\
&~ + \frac{2\sqrt{\Cscr(\nu;\mu,\Fcal,\pi)}}{1 - \gamma} \left(\|f_{\pi,\max} - \Tcal^\pi f_{\pi,\max}\|_{2,\mu} + \|f_{\pi,\min} - \Tcal^\pi f_{\pi,\min}\|_{2,\mu}\right).
\end{align}

Since \Eqref{eq:fminbd_2} holds for arbitrary on-support state-action distribution $\nu$, we take the minimal over the set of all $\{\nu:\Cscr(\nu;\mu,\Fcal,\pi) \leq C_2\}$ ($C_2$ denotes the $L^2$ concentrability threshold), and obtain  
\begin{align}
&~ f_{\pi,\max}(\pi,s_0) - f_{\pi,\min}(\pi,s_0)
\\
\leq &~ \min_{\nu:\Cscr(\nu;\mu,\Fcal,\pi) \leq C_2} \Bigg( \frac{2\sqrt{\Cscr(\nu;\mu,\Fcal,\pi)}}{1 - \gamma} \left(\|f_{\pi,\max} - \Tcal^\pi f_{\pi,\max}\|_{2,\mu} + \|f_{\pi,\min} - \Tcal^\pi f_{\pi,\min}\|_{2,\mu}\right) \\
&~ + \frac{1}{1 - \gamma} \sum_{(s,a) \in \Scal\times\Acal} (d_{\pi}\setminus\nu)(s,a) \left[\Delta f_{\pi}(s,a) - \gamma(\Pcal^\pi\Delta f_{\pi})(s,a)\right] \Bigg)
\\
\leq &~ \frac{2 \sqrt{C_2}}{1 - \gamma} \left(\|f_{\pi,\max} - \Tcal^\pi f_{\pi,\max}\|_{2,\mu} + \|f_{\pi,\min} - \Tcal^\pi f_{\pi,\min}\|_{2,\mu}\right)\\
&~ + \min_{\nu:\Cscr(\nu;\mu,\Fcal,\pi) \leq C_2} \left( \frac{1}{1 - \gamma} \sum_{(s,a) \in \Scal\times\Acal} (d_{\pi}\setminus\nu)(s,a) \left[\Delta f_{\pi}(s,a) - \gamma(\Pcal^\pi\Delta f_{\pi})(s,a)\right] \right)
\\
\leq &~ \frac{4 \sqrt{C_2} \sqrt{\varepsilon_b}}{1 - \gamma} \\
&~ + \min_{\nu:\Cscr(\nu;\mu,\Fcal,\pi) \leq C_2} \left( \frac{1}{1 - \gamma} \sum_{(s,a) \in \Scal\times\Acal} (d_{\pi}\setminus\nu)(s,a) \left[\Delta f_{\pi}(s,a) - \gamma(\Pcal^\pi\Delta f_{\pi})(s,a)\right] \right).
\end{align}

This implies the bound of $J(\pi) - J(\pihat)$,
\begin{align}
J(\pi) - J(\pihat) \leq &~ \max_{f \in \Fcal_{\pi,\varepsilon_r}}f(s_0,\pi) - \min_{f \in \Fcal_{\pihat,\varepsilon_r}}f(s_0,\pihat) + \frac{2\sqrt{\varepsilon_\Fcal}}{1 - \gamma}
\\
\leq &~ \max_{f \in \Fcal_{\pi,\varepsilon_r}}f(s_0,\pi) - \min_{f \in \Fcal_{\pi,\varepsilon_r}}f(s_0,\pi) + \frac{2\sqrt{\varepsilon_\Fcal}}{1 - \gamma}
\\
= &~ f_{\pi,\max}(\pi,s_0) - f_{\pi,\min}(\pi,s_0) + \frac{2\sqrt{\varepsilon_\Fcal}}{1 - \gamma}
\\
\leq &~ \frac{4 \sqrt{C_2} \sqrt{\varepsilon_b}}{1 - \gamma} + \frac{2\sqrt{\varepsilon_\Fcal}}{1 - \gamma} \\
&~ + \min_{\nu:\Cscr(\nu;\mu,\Fcal,\pi) \leq C_2} \left( \frac{1}{1 - \gamma} \sum_{(s,a) \in \Scal\times\Acal} (d_{\pi}\setminus\nu)(s,a) \left[\Delta f_{\pi}(s,a) - \gamma(\Pcal^\pi\Delta f_{\pi})(s,a)\right] \right). 
\end{align}
Plugging the definition of $\varepsilon_b$ (in \Eqref{eq:upbdmsbe}), we complete the proof.
\end{proof}

\subsection{Detailed Proofs for Linear Function Approximation Results}
\label{sec:info_theo_proof_linear}

\paragraph{Existence of near-optimal $\piphi$ with sufficiently large $L_2$}

We show that, if we select $L_2 = \log|\Acal| \sqrt{n}$ for $\Pi_\Phi$ (defined in Definition~\ref{def:linearmdp}), then the suboptimality of $\piphi$, the best policy in $\Pi_\Phi$, must be at most $\Ocal(\frac{\Vmax}{(1 - \gamma)\sqrt{n}})$, which is absorbed by the big-Oh in Theorem~\ref{thm:linear_bound} when  we compete with $\pi^\star$.

Consider following procedure: Let $\eta = \sqrt{\frac{\log|\Acal|}{2 \Vmax^2 T}}$. For any $t \in [T]$,
\begin{enumerate}
\item $Q_t =  \phi(\cdot,\cdot)^\T \theta_t = Q^{\pi_t}$. (We have $\| \theta_t \| \leq L_1$ by Definition~\ref{def:linearmdp} and Assumption~\ref{asm:lin_real_comp})
\item $\pi_{t + 1}(\cdot|s) \propto \exp \left(\langle \phi(s,\cdot)^\T , \eta \sum_{i = 1}^{t} \theta_t \rangle \right),~\forall s \in \Scal$.
\end{enumerate}

Then, by a standard no-regret argument\footnote{As a concrete proof, in Lemma~\ref{lem:ftlbtl2} through Theorem~\ref{thm:ewargt}, the choice of $f_t$ can be arbitrary and we do not use the fact it is learned via pessimistic policy evaluation; therefore, we can adapt the proofs by letting $f_t = Q^{\pi_t}$, and subsequently setting $\Mcal_t$ as the true MDP for all $t$. %
As a result, Theorem~\ref{thm:ewargt} directly produces the desired no-regret bound.}, we have
\begin{align}
J(\pi^\star) - \max_{t \in [T]} J(\pi_t) \leq\max_{\pi \in \Pi} J(\pi) - \frac{1}{T} \sum_{t = 1}^{T} J(\pi_{t}) \leq \Ocal \left( \frac{\Vmax}{1 - \gamma} \sqrt{\frac{\log|\Acal|}{T}} \right).
\end{align}

Selecting $T = n \log|\Acal|$ implies
\begin{align}
J(\pi^\star) - \max_{t \in [T]} J(\pi_t) \leq \frac{\Vmax}{(1 - \gamma) \sqrt{n}}.
\end{align}
Therefore, we know that one of $\pi_1, \ldots, \pi_T$ must be near-optimal. It remains to bound the norm of the coefficients of $\pi_t$ for $t \in [T]$:
\begin{align}
\left \| \eta \sum_{i = 1}^{t} \theta_t \right\|_2 \leq &~ \eta \sum_{i = 1}^{t} \left \| \theta_t \right\|_2
\\
\leq &~ \sqrt{\frac{\log|\Acal|}{2 \Vmax^2 T}} \cdot t \cdot L_1
\\
\leq &~ \frac{L_1}{\Vmax} \sqrt{T \log|\Acal|}
\\
= &~ \frac{L_1 \log|\Acal|}{\Vmax} \sqrt{n}.
\end{align}
This means if we select $L_2 = \nicefrac{L_1 \log|\Acal| \sqrt{n}}{\Vmax}$, we have $\pi_t \in \Pi_\Phi$ for all $t \in [T]$, and consequently the suboptimality of $\piphi$ must be $\Ocal(\frac{\Vmax}{(1 - \gamma)\sqrt{n}})$.

\begin{proof}[\bf\em Proof of Theorem~\ref{thm:linear_bound}]
We use $\Theta$ to denote the parameter space of $\Fcal_\Phi$, i.e., $\Fcal_\Phi = \{\phi^\T \theta: \theta \in \Theta\}$. And we also use $\Theta_{\pi,\varepsilon_r}$ to denote the version space in the parameter space accordingly, i.e., $\Theta_{\pi,\varepsilon_r} = \left\{\theta \in \Theta: \Ecal(f,\pi;\Dcal) \leq \varepsilon_r\right\}$.
Now, using the optimality of $\pihat$, we have
\begin{align}
\label{eq:total_gap}
\max_{\theta \in \Theta_{\pi,\varepsilon_r}}\phi(s_0,\pi)^\T \theta - \min_{\theta \in \Theta_{\pihat,\varepsilon_r}}\phi(s_0,\pihat)^\T \theta \leq \max_{\theta \in \Theta_{\pi,\varepsilon_r}}\phi(s_0,\pi)^\T \theta - \min_{\theta \in \Theta_{\pi,\varepsilon_r}}\phi(s_0,\pi)^\T \theta.
\end{align}
Let $\theta_{\pi,\min} \coloneqq \argmin_{\theta \in \Theta_{\pi,\varepsilon_r}} \phi(s_0,\pi)^\T \theta$ and $\theta_{\pi,\max} \coloneqq \argmax_{\theta \in \Theta_{\pi,\varepsilon_r}} \phi(s_0,\pi)^\T \theta$. By a standard telescoping argument (e.g., \citep[Lemma 1]{xie2020q}), we can obtain
\begin{align}
&~ \left|\phi(s_0,\pi)^\T \theta_{\pi,\min} - J(\pi) \right|
\\
= &~ \frac{1}{1 - \gamma} \left| \Ebb_{d_\pi} \left[\phi(s,a)^\T \theta_{\pi,\min} - \left(\Tcal^\pi \phi^\T \theta_{\pi,\min}\right)(s,a) \right] \right|
\\
\leq &~ \frac{1}{1 - \gamma} \Ebb_{d_\pi} \left|\phi(s,a)^\T \theta_{\pi,\min} - \left( \Tcal^\pi \phi^\T \theta_{\pi,\min}\right) (s,a) \right|
\\
\label{eq:telescoping}
= &~ \frac{1}{1 - \gamma} \Ebb_{d_\pi} \bigg| \phi(s,a)^\T \bigg( \underbrace{\theta_{\pi,\min} - \theta'}_{\eqqcolon \xi_{\pi,\min}}\bigg)  \bigg|,
\end{align}
where $\theta'$ must exist by the linear completeness assumption.

We now define $\Sigma_\Dcal$,
\begin{align}
\Sigma_{\Dcal} \coloneqq \E_\Dcal \left[\phi(s,a) \phi(s,a)^\T \right],
\end{align}
and $\Ecal_{\varepsilon_r}$,
\begin{align}
\Ecal_{\varepsilon_r}(s,a) \coloneqq \left|\phi(s,a)^\T \xi_{\pi,\min} \right|, ~ \forall (s,a) \in \Scal \times \Acal.
\end{align}
By definition of $\xi_{\pi,\min}$ and Theorem~\ref{thm:mspo2be_linear_emp}, we have
\begin{align}
\xi_{\pi,\min}^\T \Sigma_\Dcal \xi_{\pi,\min} \leq &~ \varepsilon_b
\\
\Longrightarrow \left\| \Sigma_\Dcal^{\nicefrac{1}{2}} \xi_{\pi,\min} \right\|_2 \leq &~ \sqrt{\varepsilon_b}.
\end{align}
Then, for any $(s,a) \in \Scal \times \Acal$,
\begin{align}
\Ecal_{\varepsilon_r}(s,a) = &~ \left|\phi(s,a)^\T \Sigma_\Dcal^{-\nicefrac{1}{2}} \Sigma_\Dcal^{\nicefrac{1}{2}}\xi_{\pi,\min} \right|
\\
\leq &~ \left\|\phi(s,a)^\T \Sigma_\Dcal^{-\nicefrac{1}{2}} \right\|_2 \left\| \Sigma_\Dcal^{\nicefrac{1}{2}}\xi_{\pi,\min} \right\|_2
\\
\label{eq:changemeasure}
\leq &~ \sqrt{\phi(s,a)^\T \Sigma_\Dcal^{-1} \phi(s,a)} \sqrt{\varepsilon_b}.
\end{align}

Plugging \Eqref{eq:changemeasure} into \Eqref{eq:telescoping}, we obtain
\begin{align}
&~ \left|\phi(s_0,\pi)^\T \theta_{\pi,\min} - J(\pi) \right|
\\
\leq &~ \frac{1}{1 - \gamma} \Ebb_{d_\pi} \left[ \Ecal_{\varepsilon_r}(s,a) \right]
\\
\label{eq:mingap}
\leq &~ \frac{\sqrt{\varepsilon_b}}{1 - \gamma} \Ebb_{d_\pi} \left[  \sqrt{\phi(s,a)^\T \Sigma_\Dcal^{-1} \phi(s,a)} \right].
\end{align}

Similarly, we also have
\begin{align}
\label{eq:maxgap}
\left|\phi(s_0,\pi)^\T \theta_{\pi,\max} - J(\pi) \right|
\leq \frac{\sqrt{\varepsilon_b}}{1 - \gamma} \Ebb_{d_\pi} \left[  \sqrt{\phi(s,a)^\T \Sigma_\Dcal^{-1} \phi(s,a)} \right].
\end{align}

Combining \Eqref{eq:total_gap}, \Eqref{eq:mingap}, and \Eqref{eq:maxgap},
\begin{align}
\max_{\theta \in \Theta_{\pi,\varepsilon_r}}\phi(s_0,\pi)^\T \theta - \min_{\theta \in \Theta_{\pihat,\varepsilon_r}}\phi(s_0,\pihat)^\T \theta \leq &~ \max_{\theta \in \Theta_{\pi,\varepsilon_r}}\phi(s_0,\pi)^\T \theta - \min_{\theta \in \Theta_{\pi,\varepsilon_r}}\phi(s_0,\pi)^\T \theta
\\
\label{eq:linearfinalgap}
\leq &~ \frac{2\sqrt{\varepsilon_b}}{1 - \gamma} \Ebb_{d_\pi} \left[  \sqrt{\phi(s,a)^\T \Sigma_\Dcal^{-1} \phi(s,a)} \right].
\end{align}

By the definition of $\varepsilon_r$, we know $\theta_\pi \in \Theta_{\pi,\varepsilon_r}$ for any $\pi \in \Pi$. This implies
\begin{align}
J(\pi) - J(\pihat) = &~ Q^\pi(s_0,\pi) - Q^\pihat(s_0,\pihat)
\\
= &~ \phi(s_0,\pi)^\T \theta_\pi - \phi(s_0,\pihat)^\T \theta_\pihat
\\
\leq &~ \max_{\theta \in \Theta_{\pi,\varepsilon_r}}\phi(s_0,\pi)^\T \theta - \min_{\theta \in \Theta_{\pihat,\varepsilon_r}}\phi(s_0,\pihat)^\T \theta
\\
\leq &~ \frac{2\sqrt{\varepsilon_b}}{1 - \gamma} \Ebb_{d_\pi} \left[  \sqrt{\phi(s,a)^\T \Sigma_\Dcal^{-1} \phi(s,a)} \right],
\end{align}
where the last inequality follows from \Eqref{eq:linearfinalgap}. Plugging the definition of $\varepsilon_b$ (defined in Theorem~\ref{thm:mspo2be_linear_emp}), we completes the proof.
\end{proof}
\section{Detailed Proofs in Section~\ref{sec:regulpi}}
\label{sec:prac_algo_proof}

\subsection{Some Lemmas}

We first introduced the necessary lemmas that used in our proofs. 

In the following lemma, we show that at every iteration $t$ of Algorithm \ref{alg:reg_pi}, the estimated Q-function ($f_t$ obtained at the step \ref{lin:reg_pes_pe} of Algorithm \ref{alg:reg_pi}) is actually the true Q-value of $\pi_t$ in a specific MDP $\Mcal_t$, denoted by $Q^{\pi_t}_{\Mcal_t}$, where dynamic of $\Mcal_t$ is same as the ground-truth MDP $\Mcal$ and the difference between the reward functions of $\Mcal$ and $\Mcal_t$ can be controlled. 

\begin{lemma}
\label{lem:plceva}
Let $f_t$ satisfies $\Ecal(f_t,\pi_t;\Dcal) \leq \varepsilon$ for some $\pi_t$. Then, there exists an MDP $\Mcal_t = (\Pcal_t, \Rcal_t)$ (the other elements of $\Mcal_t$ are same as the environment MDP $\Mcal$, and also let $R_t(s,a) = \Ebb[\Rcal_t(s,a)]$) with $\Pcal_t = \Pcal$ and $\| R_t(s,a) - R(s,a)\|_{2,\mu}^2 \leq \varepsilon_b$, such that $f_t = Q^{\pi_t}_{\Mcal_t}$, where $\varepsilon_b$ is defined in Theorem \ref{thm:mspo2be}.
\end{lemma}
\begin{proof}[\bf\em Proof of Lemma~\ref{lem:plceva}]
We can simply set $\Rcal_t = R_t$ is deterministic as,
\begin{align}
\label{eq:defrt}
R_t(s,a) \coloneqq f_t(s,a) - \gamma \Eop_{s' \sim \Pcal_t(\cdot|s,a)} \left[\sum_{a' \in \Acal} \pi_t(a'|s') f_t(s',\pi) \right].
\end{align}
Note that, this $R_t$ always exist because the definition above is equivalent to $R_t = (I - \gamma \Pcal) f_t$.

With this $\Pcal_t$ and $\Rcal_t$ ($\Mcal_t = (\Pcal_t, \Rcal_t)$), it directly implies that
\begin{align}
f_t(s,a) = R_t(s,a) + \gamma \Eop_{s' \sim \Pcal_t(\cdot|s,a)} \left[\sum_{a' \in \Acal} \pi_t(a'|s') f_t(s',\pi) \right] = (\Tcal^{\pi_t}_{\Mcal_t} f_t)(s,a),
\end{align}
which means that $f_t$ is the Q-function of $\pi_t$ in MDP $\Mcal_t$, i.e., $f_t(s,a) = Q^\pi_{\Mcal_t}(s,a)$ for any $(s,a) \in \Scal \times \Acal$.

For $\| R_t(s,a) - R(s,a)\|_{2,\mu}^2$, we have
\begin{align}
\| R_t(s,a) - R(s,a)\|_{2,\mu}^2 = &~ \left\| f_t(s,a) - \gamma \Eop_{s' \sim \Pcal_t(\cdot|s,a)} \left[\sum_{a' \in \Acal} \pi_t(a'|s') f_t(s',\pi) \right] - R(s,a)\right\|_{2,\mu}^2  \tag{by \Eqref{eq:defrt}}
\\
= &~ \left\| f_t(s,a) - \gamma \Eop_{s' \sim \Pcal(\cdot|s,a)} \left[\sum_{a' \in \Acal} \pi_t(a'|s') f_t(s',\pi) \right] - R(s,a)\right\|_{2,\mu}^2
\\
= &~ \left\| f_t- \Tcal^{\pi_t} f_t \right\|_{2,\mu}^2 
\\
\leq &~ \varepsilon_b,
\end{align}
where the last inequality follows from the definition of $\varepsilon_b$ (in Appendix \ref{sec:PEresults_generalFA}). This completes the proof.
\end{proof}

\begin{definition}
\label{def:defreg}
Consider following procedure: for any $t \in [T]$
\begin{enumerate}
\item $f_t = \argmin_{f \in \Fcal} \left(f(s_0,\pi_t) + \lambda \Ecal(f,\pi_t;\Dcal)\right)$ (step \ref{lin:reg_pes_pe} of Algorithm \ref{alg:reg_pi})
\item $\pi_{t + 1}(a|s) \propto \pi_{t}(a|s) \exp \left( \eta f_t(s,a) \right),~\forall s,a \in \Scal \times \Acal$.
\end{enumerate}
Let $J_{\Mcal}(\pi)$ denotes the policy return under MDP $\Mcal$. Then, we define the total regret of the above procedure as
\begin{align}
\regret_T \coloneqq \max_{\pi: \Scal \to \Delta(\Acal)} \sum_{i = 1}^{T} J_{\Mcal_t}(\pi) - J_{\Mcal_t}(\pi_{t}).
\end{align}
\end{definition}

Over this section, we define $\ell_s(\pi)$ as
\begin{align}
\ell_s(\pi) \coloneqq \frac{1}{\eta} \sum_{a \in \Acal} \pi(a|s) \log \pi(a|s).
\end{align}

\begin{lemma}
\label{lem:ftlbtl2}
For any $\pi: \Scal \to \Delta(\Acal)$ and $s \in \Scal$,
\begin{align}
\label{eq:ftlbtl2}
\sum_{t = 1}^{T} \langle \pi_{t + 1}(\cdot|s), f_t(s,\cdot) \rangle - \ell_s(\pi_1) \geq \sum_{t = 1}^{T} \langle \pi(\cdot|s), f_t(s,\cdot) \rangle - \ell_s(\pi).
\end{align}
\end{lemma}
\begin{proof}[\bf\em Proof of Lemma~\ref{lem:ftlbtl2}]
We establish our proof by induction. The case of $T = 0$ holds as $\pi_1$ is the uniform policy. We assume \Eqref{eq:ftlbtl2} holds at $T = T'$, then for the case of $T' + 1$, we have the follows for any $\pi \in \Pi$ and $s \in \Scal$
\begin{align}
&~ \sum_{t = 1}^{T' + 1} \langle \pi(\cdot|s), f_t(s,\cdot) \rangle - \ell_s(\pi)
\\
\overset{\text{(a)}}{\leq} &~ \sum_{t = 1}^{T' + 1} \langle \pi_{T' + 2}(\cdot|s), f_t(s,\cdot) \rangle - \ell_s(\pi_{T' + 2}(\cdot|s))
\\
= &~ \sum_{t = 1}^{T'} \langle \pi_{T' + 2}(\cdot|s), f_t(s,\cdot) \rangle - \ell_s(\pi_{T' + 2}(\cdot|s)) + \langle \pi_{T' + 2}(\cdot|s), f_{T' + 1}(s,\cdot) \rangle
\\
\overset{\text{(b)}}{\leq} &~ \sum_{t = 1}^{T'} \langle \pi_{t + 1}(\cdot|s), f_t(s,\cdot) \rangle - \ell_s(\pi_1) + \langle \pi_{T' + 2}(\cdot|s), f_{T' + 1}(s,\cdot) \rangle
\\
= &~ \sum_{t = 1}^{T' + 1} \langle \pi_{t + 1}(\cdot|s), f_t(s,\cdot) \rangle - \ell_s(\pi_1),
\end{align}
where (a) follows from the fact that $\pi_{T' + 2}(\cdot|s)$ is the global maximizer of $\sum_{t = 1}^{T' + 1} \langle \pi(\cdot|s), f_t(s,\cdot) \rangle - \ell_s(\pi)$, and (b) uses the induction hypothesis that \Eqref{eq:ftlbtl2} holds at $T = T'$. This completes the proof. 
\end{proof}

\begin{lemma}
\label{lem:ftlbtl}
For any $\pi: \Scal \to \Delta(\Acal)$ and $s \in \Scal$,
\begin{align}
\sum_{t = 1}^{T} \langle \pi(\cdot|s) - \pi_{t}(\cdot|s), f_t(s,\cdot) \rangle \leq \sum_{t = 1}^{T} \langle \pi_{t + 1}(\cdot|s) - \pi_{t}(\cdot|s), f_t(s,\cdot) \rangle - \ell_s(\pi_1).
\end{align}
\end{lemma}
\begin{proof}[{\bf\em Proof of Lemma~\ref{lem:ftlbtl}}]
We use the result of Lemma~\ref{lem:ftlbtl2} to establish the proof.
\begin{align}
&~ \sum_{t = 1}^{T} \langle \pi(\cdot|s) - \pi_{t}(\cdot|s), f_t(s,\cdot) \rangle
\\
\leq &~ \sum_{t = 1}^{T} \langle \pi(\cdot|s), f_t(s,\cdot) \rangle - \ell_s(\pi) + \ell_s(\pi) - \sum_{t = 1}^{T} \langle \pi_{t}(\cdot|s), f_t(s,\cdot) \rangle
\\
\leq &~ \sum_{t = 1}^{T} \langle \pi_{t + 1}(\cdot|s), f_t(s,\cdot) \rangle - \ell_s(\pi_1) - \sum_{t = 1}^{T} \langle \pi_{t}(\cdot|s), f_t(s,\cdot) \rangle + \ell_s(\pi)\tag{by Lemma~\ref{lem:ftlbtl2}}
\\
\leq &~ \sum_{t = 1}^{T} \langle \pi_{t + 1}(\cdot|s) - \pi_{t}(\cdot|s), f_t(s,\cdot) \rangle - \ell_s(\pi_1). \tag*{\qedhere}
\end{align}
\end{proof}

\begin{lemma}
\label{lem:persrgt}
For any $\pi: \Scal \to \Delta(\Acal)$ and $s \in \Scal$,
\begin{align}
\sum_{t = 1}^{T} \langle \pi(\cdot|s) - \pi_{t}(\cdot|s), f_t(s,\cdot) \rangle \leq 2 \Vmax \sqrt{2 \log|\Acal| T},
\end{align}
if we take $\eta = \sqrt{\frac{\log|\Acal|}{2 \Vmax^2 T}}$.
\end{lemma}
\begin{proof}[\bf\em Proof of Lemma~\ref{lem:persrgt}]
We define $\Lcal_{s,t}$ as
\begin{align}
\Lcal_{s,t}(\pi) \coloneqq \sum_{t' = 1}^{t} \langle \pi(\cdot|s), f_{t'}(s,\cdot) \rangle - \ell_s(\pi).
\end{align}
Let $B_{\Lcal_{s,t}}(\cdot\|\cdot)$ and $B_{\ell_{s}}(\cdot\|\cdot)$ be the Bregman divergence w.r.t.~$\Lcal_{s,t}(\cdot)$ and $\ell_{s}(\cdot)$, then we have
\begin{align}
\Lcal_{s,t}(\pi_t) \overset{\text{(a)}}{=} &~ \Lcal_{s,t}(\pi_{t + 1}) + \langle \pi_t(\cdot|s) - \pi_{t + 1}(\cdot|s), \nabla \Lcal_{s,t}(\pi)|_{\pi = \pi_{t + 1}} \rangle + B_{\Lcal_{s,t}}(\pi_t(\cdot|s) \| \pi_{t + 1}(\cdot|s))
\\
\overset{\text{(b)}}{\leq} &~ \Lcal_{s,t}(\pi_{t + 1}) + B_{\Lcal_{s,t}}(\pi_t(\cdot|s) \| \pi_{t + 1}(\cdot|s))
\\
\overset{\text{(c)}}{=} &~ \Lcal_{s,t}(\pi_{t + 1}) - B_{\ell_{s}}(\pi_t(\cdot|s) \| \pi_{t + 1}(\cdot|s)),
\end{align}
where (a) is obtained the the definition of Bregman divergence, (b) follows from the fact that $\pi_{t+1}$ maximizes $\Lcal_{s,t}(\pi)$ by definition, and (c) is because $\Lcal_{s,t}(\pi) + \ell_s(\pi)$ is linear, which does not affect the Bregman divergence.

By reordering the inequality above, we obtain
\begin{align}
B_{\ell_{s}}(\pi_t(\cdot|s) \| \pi_{t + 1}(\cdot|s)) \leq &~ \left(\Lcal_{s,t}(\pi_{t + 1}) - \Lcal_{s,t}(\pi_{t})\right)
\\
= &~ \left[\Lcal_{s,t - 1}(\pi_{t + 1}) - \Lcal_{s,t - 1}(\pi_{t}) + \langle \pi_{t + 1}(\cdot|s) - \pi_{t}(\cdot|s), f_t(s,\cdot) \rangle\right]
\\
\label{eq:bdbls}
\leq &~ \langle \pi_{t + 1}(\cdot|s) - \pi_{t}(\cdot|s), f_t(s,\cdot) \rangle,
\end{align}
where the last inequality is because $\pi_{t}$ maximizes $\Lcal_{s,t - 1}(\cdot)$.

By applying the Taylor expansion and the mean-value theorem on $B_{\ell_{s}}$, we can rewrite $B_{\ell_{s}}$ as 
\begin{align}
\label{eq:expressbls}
B_{\ell_{s}}(\pi_t(\cdot|s) \| \pi_{t + 1}(\cdot|s)) = &~ \frac{1}{2} \|\pi_t(\cdot|s) - \pi_{t + 1}(\cdot|s)\|_{(H \ell_{s})(\pi_{t'})}^2\\
\coloneqq &~ \frac{1}{2} (\pi_t(\cdot|s) - \pi_{t + 1}(\cdot|s))^\T [(H \ell_{s})(\pi_{t'})] (\pi_t(\cdot|s) - \pi_{t + 1}(\cdot|s)),
\end{align}
where $\pi_t' = \alpha \pi_t + (1 - \alpha) \pi_{t + 1}$ for some $\alpha \in [0,1]$, and $H \ell_{s}$ denotes the Hessian matrix of $\ell_{s}$.

We now bound $\langle \pi_{t + 1}(\cdot|s) - \pi_{t}(\cdot|s), f_t(s,\cdot) \rangle$ using the results above. By the generalized Cauchy-Schwarz theorem,
\begin{align}
\langle \pi_{t + 1}(\cdot|s) - \pi_{t}(\cdot|s), f_t(s,\cdot) \rangle \leq &~ \left\|\pi_{t + 1}(\cdot|s) - \pi_{t}(\cdot|s)\right\|_{(H \ell_{s})(\pi_{t'})} \left\|f_t(s,\cdot)\right\|_{(H \ell_{s})^{-1}(\pi_{t'})}
\\
= &~ \sqrt{2 B_{\ell_{s}}(\pi_t(\cdot|s) \| \pi_{t + 1}(\cdot|s))} \left\|f_t(s,\cdot)\right\|_{(H \ell_{s})^{-1}(\pi_{t'})}  \tag{by \Eqref{eq:expressbls}}
\\
\leq &~ \sqrt{2 B_{\ell_{s}}(\pi_t(\cdot|s) \| \pi_{t + 1}(\cdot|s))} \sqrt{\eta} \left\|f_t(s,\cdot)\right\|_{\infty}
\\
\leq &~ \sqrt{2 \langle \pi_{t + 1}(\cdot|s) - \pi_{t}(\cdot|s), f_t(s,\cdot) \rangle} \sqrt{\eta} \left\|f_t(s,\cdot)\right\|_{\infty} \tag{by \Eqref{eq:bdbls}}
\\
\leq &~ \Vmax \sqrt{2 \eta \langle \pi_{t + 1}(\cdot|s) - \pi_{t}(\cdot|s), f_t(s,\cdot) \rangle}
\\
\Longrightarrow \langle \pi_{t + 1}(\cdot|s) - \pi_{t}(\cdot|s), f_t(s,\cdot)\rangle \leq &~ 2 \eta \Vmax^2.
\end{align}

As $\pi_1$ is the uniform policy, we know $\ell_s(\pi_1) = -\log|\Acal| / \eta$. Therefore, by applying Lemma~\ref{lem:ftlbtl}, we obtain
\begin{align}
\sum_{t = 1}^{T} \langle \pi - \pi_{t}(\cdot|s), f_t(s,\cdot) \rangle \leq &~ \sum_{t = 1}^{T} \langle \pi_{t + 1}(\cdot|s) - \pi_{t}(\cdot|s), f_t(s,\cdot) \rangle - \ell_s(\pi_1)
\\
\leq &~ 2 \eta \Vmax^2 T + \frac{\log|\Acal|}{\eta}
\\
= &~ 2 \Vmax \sqrt{2 T \log|\Acal|},
\end{align}
where the last step is attained by taking $\eta = \sqrt{\frac{\log|\Acal|}{2 \Vmax^2 T}}$.
\end{proof}

\begin{theorem}
\label{thm:ewargt}
Let $\pitilde = \argmax_{\pi: \Scal \to \Delta(\Acal)} \sum_{t = 1}^{T} J_{\Mcal_t}(\pi) - J_{\Mcal_t}(\pi_{t})$ and $\eta = \sqrt{\frac{\log|\Acal|}{2 \Vmax^2 T}}$, we have
\begin{align}
\regret_T \leq \frac{2 \Vmax \sqrt{2 T \log|\Acal|}}{1 - \gamma}.
\end{align}
\end{theorem}
\begin{proof}[{\bf\em Proof of Theorem~\ref{thm:ewargt}}]
Using the performance difference lemma, we have
\begin{align}
\regret_T \coloneqq &~ \sum_{t = 1}^{T} J_{\Mcal_t}(\pitilde) - J_{\Mcal_t}(\pi_{t})
\\
= &~ \frac{1}{1 - \gamma} \sum_{t = 1}^{T} \Eop_{s \sim d_{\pitilde,\Mcal_t}} \left[Q^{\pi_t}_{\Mcal_t}(s,\pitilde) - Q^{\pi_t}_{\Mcal_t}(s,\pi_t) \right]
\tag{by performance difference lemma \citep{kakade2002approximately}}
\\
= &~ \frac{1}{1 - \gamma} \sum_{t = 1}^{T} \Eop_{s \sim d_{\pitilde,\Mcal_t}} \left[f_t(s,\pitilde) - f_t(s,\pi_t) \right]  \tag{by Lemma~\ref{lem:plceva}}
\\
= &~ \frac{1}{1 - \gamma} \sum_{t = 1}^{T} \Eop_{s \sim d_{\pitilde,\Mcal_t}} \left[\langle \pitilde(\cdot|s) - \pi_t(\cdot|s), f_t(s,\cdot) \rangle \right]
\end{align}

By Lemma~\ref{lem:plceva} and its proof, we know the dynamics of $\Mcal_t, t \in [T]$ are identical (same as that of the true environment MDP). Let $d_{\pitilde} = d_{\pitilde,\Mcal_t}$ which holds for any $t \in [T]$, and we have
\begin{align}
\regret_T = &~ \frac{1}{1 - \gamma} \sum_{t = 1}^{T} \Eop_{s \sim d_{\pitilde}} \left[\langle \pitilde(\cdot|s) - \pi_t(\cdot|s), f_t(s,\cdot) \rangle \right]
\\
= &~ \frac{1}{1 - \gamma} \Eop_{s \sim d_{\pitilde}} \left[\sum_{t = 1}^{T} \langle \pitilde(\cdot|s) - \pi_t(\cdot|s), f_t(s,\cdot) \rangle \right]
\\
\leq &~ \frac{1}{1 - \gamma} \Eop_{s \sim d_{\pitilde}} \left[2 \Vmax \sqrt{2 T \log|\Acal|} \right]  \tag{by Lemma~\ref{lem:persrgt}}
\\
= &~ \frac{2 \Vmax \sqrt{2 T \log|\Acal|}}{1 - \gamma}.
\end{align}
This completes the proof.
\end{proof}

\subsection{Proof of Theorem~\ref{thm:opterr_regpi}}

\begin{lemma}
\label{lem:vlduplowerbd}
For any $\pi \in \Pi$, we have,
\begin{align}
\min_{f \in \Fcal}(f(s_0,\pi) + \lambda \Ecal(f,\pi;\Dcal)) \leq J(\pi) + \frac{\sqrt{\varepsilon_\Fcal}}{1 - \gamma} + \lambda \varepsilon_r,
\end{align}
and,
\begin{align}
\max_{f \in \Fcal}(f(s_0,\pi) - \lambda \Ecal(f,\pi;\Dcal)) \geq J(\pi) - \frac{\sqrt{\varepsilon_\Fcal}}{1 - \gamma} - \lambda \varepsilon_r,
\end{align}
where $\varepsilon_r$ is defined in \Eqref{eq:upbdespr}, i.e.,
\begin{align}
\varepsilon_r \coloneqq \frac{139 \Vmax^2 \log \frac{|\Fcal||\Pi|}{\delta}}{n} + 39 \varepsilon_{\Fcal}.
\end{align}
\end{lemma}
\begin{proof}[\bf\em Proof of Lemma~\ref{lem:vlduplowerbd}]
For any $\pi \in \Pi$, let 
\begin{align}
f_\pi \coloneqq \argmin_{f \in \Fcal}\sup_{\text{admissible }\nu} \left\|f - \Tcal^\pi f\right\|_{2,\nu}^2.
\end{align}
We know $\left\|f_\pi - \Tcal^\pi f_\pi\right\|_{2,\nu}^2 \leq \varepsilon_\Fcal$ for any admissible $\nu$, which implies $\|f_\pi - \Tcal^\pi f_\pi \|_{2,d_\pi} \leq \sqrt{\varepsilon_\Fcal}$.
Then, %
\begin{align}
J(\pi) = &~ J(\pi) - (f_\pi(s_0,\pi) - \lambda \Ecal(f_\pi,\pi;\Dcal)) + (f_\pi(s_0,\pi) - \lambda \Ecal(f_\pi,\pi;\Dcal))
\\
= &~ (f_\pi(s_0,\pi) - \lambda \Ecal(f_\pi,\pi;\Dcal)) + (J(\pi) - f_\pi(s_0,\pi)) + \lambda \Ecal(f_\pi,\pi;\Dcal)
\\
\leq &~ (f_\pi(s_0,\pi) - \lambda \Ecal(f_\pi,\pi;\Dcal)) + \frac{\|f_\pi - \Tcal^\pi f_\pi \|_{2,d_\pi}}{1 - \gamma} + \lambda \Ecal(f_\pi,\pi;\Dcal)
\tag{by \citet[Lemma 1]{xie2020q}}
\\
\leq &~ (f_\pi(s_0,\pi) - \lambda \Ecal(f_\pi,\pi;\Dcal)) + \frac{\sqrt{\varepsilon_\Fcal}}{1 - \gamma} + \lambda \varepsilon_r
\tag{by Theorem~\ref{thm:version_space}}
\\
\leq &~ \max_{f \in \Fcal}(f(s_0,\pi) - \lambda \Ecal(f,\pi;\Dcal)) + \frac{\sqrt{\varepsilon_\Fcal}}{1 - \gamma} + \lambda \varepsilon_r
\end{align}
and
\begin{align}
J(\pi) = &~ J(\pi) - (f_\pi(s_0,\pi) + \lambda \Ecal(f_\pi,\pi;\Dcal)) + (f_\pi(s_0,\pi) + \lambda \Ecal(f_\pi,\pi;\Dcal))
\\
= &~ (f_\pi(s_0,\pi) + \lambda \Ecal(f_\pi,\pi;\Dcal)) + (J(\pi) - f_\pi(s_0,\pi)) - \lambda \Ecal(f_\pi,\pi;\Dcal)
\\
\geq &~ (f_\pi(s_0,\pi) + \lambda \Ecal(f_\pi,\pi;\Dcal)) - \frac{\|f_\pi - \Tcal^\pi f_\pi \|_{2,d_\pi}}{1 - \gamma} - \lambda \Ecal(f_\pi,\pi;\Dcal)
\tag{by \citet[Lemma 1]{xie2020q}}
\\
\geq &~ (f_\pi(s_0,\pi) + \lambda \Ecal(f_\pi,\pi;\Dcal)) - \frac{\sqrt{\varepsilon_\Fcal}}{1 - \gamma} - \lambda \varepsilon_r
\tag{by Theorem~\ref{thm:version_space}}
\\
\geq &~ \min_{f \in \Fcal}(f(s_0,\pi) + \lambda \Ecal(f,\pi;\Dcal)) - \frac{\sqrt{\varepsilon_\Fcal}}{1 - \gamma} - \lambda \varepsilon_r.
\end{align}
This completes the proof.
\end{proof}

\begin{proof}[\bf\em Proof of Theorem~\ref{thm:opterr_regpi}]
We use $\Mcal_t$ to denote the corresponding MDP of $f_t$ (see Lemma~\ref{lem:plceva}). For any $\pi \in \Pi$, let 
\begin{align}
f_\pi \coloneqq \argmin_{f \in \Fcal}\sup_{\text{admissible }\nu} \left\|f - \Tcal^\pi f\right\|_{2,\nu}^2.
\end{align}

By Lemma~\ref{lem:vlduplowerbd}, we know
\begin{align}
J(\pi_t) \geq &~ \min_{f \in \Fcal}(f(s_0,\pi_t) + \lambda \Ecal(f,\pi_t;\Dcal)) - \frac{\sqrt{\varepsilon_\Fcal}}{1 - \gamma} - \lambda \varepsilon_r
\\
= &~ f_t(s_0,\pi_t) + \lambda \Ecal(f_t,\pi_t;\Dcal) - \frac{\sqrt{\varepsilon_\Fcal}}{1 - \gamma} - \lambda \varepsilon_r
\\
\geq &~ f_t(s_0,\pi_t) - \frac{\sqrt{\varepsilon_\Fcal}}{1 - \gamma} - \lambda \varepsilon_r
\\
\label{eq:lbjpit}
= &~ J_{\Mcal_t}(\pi_t) - \frac{\sqrt{\varepsilon_\Fcal}}{1 - \gamma} - \lambda \varepsilon_r.
\end{align}

Now, we have
\begin{align}
J(\pi) - J(\bar\pi) = &~ \frac{1}{T} \sum_{t = 1}^{T} \left(J(\pi) - J(\pi_t) \right)
\\
\leq &~ \frac{1}{T} \sum_{t = 1}^{T} \left(J(\pi) - J_{\Mcal_t}(\pi) + J_{\Mcal_t}(\pi) - J_{\Mcal_t}(\pi_t) \right) + \frac{\sqrt{\varepsilon_\Fcal}}{1 - \gamma} + \lambda \varepsilon_r
\tag{by \Eqref{eq:lbjpit}}
\\
\leq &~ \frac{1}{T} \sum_{t = 1}^{T} \left(J_{\Mcal_t}(\pi) - J_{\Mcal_t}(\pi_t) \right) + \frac{1}{T} \sum_{t = 1}^{T} \left(J(\pi) - J_{\Mcal_t}(\pi) \right) + \frac{\sqrt{\varepsilon_\Fcal}}{1 - \gamma} + \lambda \varepsilon_r
\\
\leq &~ \frac{2 \Vmax}{1 - \gamma}\sqrt{\frac{ 2 \log|\Acal|}{T}} + \frac{1}{T} \sum_{t = 1}^{T} \left(J(\pi) - J_{\Mcal_t}(\pi) \right) + \frac{\sqrt{\varepsilon_\Fcal}}{1 - \gamma} + \lambda \varepsilon_r. \tag{by Lemma~\ref{thm:ewargt}}
\end{align}

We now provide the bound on $J(\pi) - J_{\Mcal_t}(\pi)$ for any $t \in [T]$.
By a standard telescoping argument (e.g., \citep[Lemma 1]{xie2020q}), we can obtain
\begin{align}
&~ J(\pi) - J_{\Mcal_t}(\pi) \\
= &~ Q^\pi(s_0,\pi) - J_{\Mcal_t}(\pi)
\\
= &~ \frac{1}{1 - \gamma} \left| \Ebb_{\mu} \left[ \nicefrac{\nu}{\mu} \cdot \left(Q^\pi - \Tcal^\pi_{\Mcal_t} Q^\pi \right)\right] + \Ebb_{d_{\pi}} \left[Q^\pi - \Tcal^\pi_{\Mcal_t} Q^\pi \right]  - \Ebb_{\nu} \left[Q^\pi - \Tcal^\pi_{\Mcal_t} Q^\pi \right] \right|
\\
\leq &~ \underbrace{\frac{1}{1 - \gamma} \left| \Ebb_{\mu} \left[ \nicefrac{\nu}{\mu} \cdot \left(Q^\pi - \Tcal^\pi_{\Mcal_t} Q^\pi \right)\right]\right|}_{\text{(I)}} + \underbrace{\frac{1}{1 - \gamma} \left|\Ebb_{d_{\pi}} \left[Q^\pi - \Tcal^\pi_{\Mcal_t} Q^\pi \right]  - \Ebb_{\nu} \left[Q^\pi - \Tcal^\pi_{\Mcal_t} Q^\pi \right] \right|}_{\text{(II)}}
\end{align}
where $\nu$ is an arbitrary on-support state-action distribution. We now discuss these two terms above separately. Note that, the $d_{\pi}$ we used above is defined to be the distribution under the true environment MDP $\Mcal$, which is equal to $d_{\pi,\Mcal_t}$ as we define $\Mcal_t$ to have the same dynamic as $\Mcal_t$ (details in Lemma~\ref{lem:plceva}).

For the term (I),
\begin{align}
\text{(I)} = &~ \frac{1}{1 - \gamma} \left| \Ebb_{\mu} \left[ \nicefrac{\nu}{\mu} \cdot \left(Q^\pi - \Tcal^\pi_{\Mcal_t} Q^\pi \right)\right] \right|
\\
\leq &~ \frac{1}{1 - \gamma} \|Q^\pi - \Tcal^\pi_{\Mcal_t} Q^\pi\|_{2,\nu},
\end{align}
because of the Cauchy-Schwarz inequality for random variables ($|\Ebb[XY]| \leq \sqrt{\Ebb[X^2]\Ebb[Y^2]}$).

We define $(d_{\pi}\setminus\nu)$ as $(d_{\pi}\setminus\nu)(s,a) \coloneqq \max(d_\pi(s,a) - \nu(s,a),0)$ for any $(s,a) \in \Scal \times \Acal$. Then, for the term (II)
\begin{align}
\text{(II)} = &~ \frac{1}{1 - \gamma} \left|\Ebb_{d_{\pi}} \left[Q^\pi - \Tcal^\pi_{\Mcal_t} Q^\pi \right]  - \Ebb_{\nu} \left[Q^\pi - \Tcal^\pi_{\Mcal_t} Q^\pi \right] \right|
\\
= &~ \frac{1}{1 - \gamma} \left|\sum_{(s,a) \in \Scal \times \Acal} \left[d_{\pi}(s,a) - \nu(s,a)\right] \left[Q^\pi(s,a) - (\Tcal^\pi_{\Mcal_t} Q^\pi)(s,a)\right]  \right|
\\
= &~ \frac{1}{1 - \gamma} \left|\sum_{(s,a) \in \Scal \times \Acal} \1(d_{\pi}(s,a) \geq \nu(s,a)) \left[d_{\pi}(s,a) - \nu(s,a)\right] \left[Q^\pi(s,a) - (\Tcal^\pi_{\Mcal_t} Q^\pi)(s,a)\right]  \right|
\\
&~ + \frac{1}{1 - \gamma} \left|\sum_{(s,a) \in \Scal \times \Acal} \1(\nu(s,a) > d_{\pi}(s,a)) \left[\nu(s,a) - d_{\pi}(s,a)\right] \left[Q^\pi(s,a) - (\Tcal^\pi_{\Mcal_t} Q^\pi)(s,a)\right]  \right|
\\
\leq &~ \frac{1}{1 - \gamma} \left|\sum_{(s,a) \in \Scal\times\Acal} (d_{\pi}\setminus\nu)(s,a) \left[Q^\pi(s,a) - (\Tcal^\pi_{\Mcal_t} Q^\pi)(s,a)\right] \right|
\\
&~ + \frac{1}{1 - \gamma} \sum_{(s,a) \in \Scal \times \Acal} \1(\nu(s,a) > d_{\pi}(s,a)) \left[\nu(s,a) - d_{\pi}(s,a)\right] \left|Q^\pi(s,a) - (\Tcal^\pi_{\Mcal_t} Q^\pi)(s,a)\right|
\\
\leq &~ \frac{1}{1 - \gamma} \left|\sum_{(s,a) \in \Scal\times\Acal} (d_{\pi}\setminus\nu)(s,a) \left[Q^\pi(s,a) - (\Tcal^\pi_{\Mcal_t} Q^\pi)(s,a)\right] \right| + \frac{1}{1 - \gamma} \Ebb_{\nu} \left[ \left|Q^\pi - \Tcal^\pi_{\Mcal_t} Q^\pi \right|\right]
\\
\leq &~ \frac{1}{1 - \gamma} \underbrace{\left|\sum_{(s,a) \in \Scal\times\Acal} (d_{\pi}\setminus\nu)(s,a) \left[Q^\pi(s,a) - (\Tcal^\pi_{\Mcal_t} Q^\pi)(s,a)\right] \right|}_{\text{(IIa)}} + \frac{1}{1 - \gamma} \|Q^\pi - \Tcal^\pi_{\Mcal_t} Q^\pi\|_{2,\nu},
\end{align}
where the second inequality follows from the fact of $\nu(s,a) \geq \1(\nu(s,a) > d_{\pi}(s,a))[\nu(s,a) - d_{\pi}(s,a)]$ for any $(s,a) \in \Scal \times \Acal$, and the last inequality uses the Cauchy-Schwarz inequality for random variables (similar to the argument about the term (I)).

We now discuss the term (IIa),
\begin{align}
\text{(IIa)} = &~ \left|\sum_{(s,a) \in \Scal\times\Acal} (d_{\pi}\setminus\nu)(s,a) \left[Q^\pi(s,a) - (\Tcal^\pi_{\Mcal_t} Q^\pi)(s,a)\right] \right|
\\
= &~ \left|\sum_{(s,a) \in \Scal\times\Acal} (d_{\pi}\setminus\nu)(s,a) \left[Q^\pi(s,a) - R_{\Mcal_t}(s,a) - \gamma (\Pcal^\pi_{\Mcal_t} Q^\pi)(s,a)\right] \right|
\\
= &~ \left|\sum_{(s,a) \in \Scal\times\Acal} (d_{\pi}\setminus\nu)(s,a) \left[Q^\pi(s,a) - (f_t(s,a) - \gamma (\Pcal^{\pi_t}f_t)(s,a)) - \gamma (\Pcal^\pi Q^\pi)(s,a)\right] \right|
\\
= &~ \left|\sum_{(s,a) \in \Scal\times\Acal} (d_{\pi}\setminus\nu)(s,a) \left[f_t(s,a) - R(s,a) - \gamma (\Pcal^{\pi_t} f_t)(s,a)\right] \right|
\\
= &~ \left|\sum_{(s,a) \in \Scal\times\Acal} (d_{\pi}\setminus\nu)(s,a) \left[f_t(s,a) - (\Tcal^{\pi_t} f_t)(s,a)\right] \right|,
\end{align}
where the third equation follows from the definition of $R_{\Mcal_t}(s,a) \coloneqq f_t(s,a) - \gamma (\Pcal^{\pi_t}_{\Mcal_t} f_t)(s,a) = f_t(s,a) - \gamma (\Pcal^{\pi_t} f_t)(s,a)$ (refer to Lemma~\ref{lem:plceva} and its proof).

Thus,
\begin{align}
\text{(II)} \leq \left|\sum_{(s,a) \in \Scal\times\Acal} (d_{\pi}\setminus\nu)(s,a) \left[f_t(s,a) - (\Tcal^{\pi_t} f_t)(s,a)\right] \right| + \frac{1}{1 - \gamma} \|Q^\pi - \Tcal^\pi_{\Mcal_t} Q^\pi\|_{2,\nu}.
\end{align}

Combining the bounds of both term (I) and term (II), we have
\begin{align}
&~ Q^\pi(\pi,s_0) - J_{\Mcal_t}(\pi) \\
\label{eq:fminbd_opt3}
\leq &~ \frac{2}{1 - \gamma} \|Q^\pi - \Tcal^\pi_{\Mcal_t} Q^\pi\|_{2,\nu} + \left|\sum_{(s,a) \in \Scal\times\Acal} (d_{\pi}\setminus\nu)(s,a) \left[f_t(s,a) - (\Tcal^{\pi_t} f_t)(s,a)\right] \right|.
\end{align}

For the term $\|Q^\pi - \Tcal^\pi_{\Mcal_t} Q^\pi\|_{2,\nu}$,
\begin{align}
\left\|Q^\pi - \Tcal^\pi_{\Mcal_t} Q^\pi\right\|_{2,\nu} = &~ \left\|Q^\pi - R_t - \gamma \Pcal^\pi_{\Mcal_t} Q^\pi\right\|_{2,\nu}
\\
= &~ \left\|Q^\pi - R_t - \gamma \Pcal^\pi Q^\pi\right\|_{2,\nu} \tag{by Lemma~\ref{lem:plceva}}
\\
= &~ \left\|R - R_t\right\|_{2,\nu}
\\
= &~ \left\|R + \gamma \Pcal^{\pi_t} f_t - f_t \right\|_{2,\nu} \tag{by Lemma~\ref{lem:plceva}}
\\
= &~ \left\|\Tcal^{\pi_t} f_t - f_t \right\|_{2,\nu}
\\
\leq &~ \sqrt{\Cscr(\nu;\mu,\Fcal,\pi_t)} \left\|\Tcal^{\pi_t} f_t - f_t \right\|_{2,\mu}
\\
\leq &~ \sqrt{\Cscr(\nu;\mu,\Fcal,\pi_t)} (\sqrt{\varepsilon_b} + \sqrt{V_{\max}/\lambda}), 
\end{align}
where the last step is obtained by the following argument:
\begin{align}
f_t(s_0,\pi_t) + \lambda \Ecal(f_t,\pi_t;\Dcal) = &~ \min_{f \in \Fcal} (f(s_0,\pi_t) + \lambda \Ecal(f,\pi_t;\Dcal))
\\
\leq &~ f_{\pi_t}(s_0,\pi_t) + \lambda \Ecal(f_{\pi_t},\pi_t;\Dcal)
\\
\leq &~ \Vmax + \lambda \varepsilon_r.
\tag{by Theorem~\ref{thm:version_space}}
\\
\Longrightarrow \Ecal(f_{t},\pi_t;\Dcal) \leq &~ \varepsilon_r + \frac{\Vmax}{\lambda}
\end{align}
Then, applying Theorem~\ref{thm:mspo2be}, we transfer the bound on $\Ecal(f_{t},\pi_t;\Dcal)$ to the bound on $\|\Tcal^{\pi_t} f_t - f_t \|_{2,\mu}$.

Since \Eqref{eq:fminbd_opt3} holds for arbitrary on-support state-action distribution $\nu$, we take the minimal over the set of all $\{\nu:\Cscr(\nu;\mu,\Fcal,\pi_t) \leq C_{2,t}\}$ ($C_{2,t}$ denotes the $L^2$ concentrability threshold), and obtain  
\begin{align}
&~ Q^\pi(s_0,\pi) - J_{\Mcal_t}(\pi)
\\
\leq &~ \min_{\nu:\Cscr(\nu;\mu,\Fcal,\pi_t) \leq C_{2,t}} \Bigg( \frac{2\Cscr(\nu;\mu,\Fcal,\pi_t)(\sqrt{\varepsilon_b} + \sqrt{V_{\max}/\lambda})}{1 - \gamma}
\\
&~ + \left|\sum_{(s,a) \in \Scal\times\Acal} (d_{\pi}\setminus\nu)(s,a) \left[f_t(s,a) - (\Tcal^{\pi_t} f_t)(s,a)\right] \right| \Bigg)
\\
\leq &~ \frac{2 \sqrt{C_{2,t}}(\sqrt{\varepsilon_b} + \sqrt{V_{\max}/\lambda}) }{1 - \gamma} + \min_{\nu:\Cscr(\nu;\mu,\Fcal,\pi_t) \leq C_{2,t}} \left|\sum_{(s,a) \in \Scal\times\Acal} (d_{\pi}\setminus\nu)(s,a) \left[f_t(s,a) - (\Tcal^{\pi_t} f_t)(s,a)\right] \right|.
\end{align}

Therefore, we complete the proof as follows.
\begin{align}
&~ J(\pi) - J(\bar\pi) 
\\
\leq &~ \frac{2 \Vmax}{1 - \gamma}\sqrt{\frac{ 2 \log|\Acal|}{T}} + \frac{1}{T} \sum_{t = 1}^{T} \left(J(\pi) - J_{\Mcal_t}(\pi) \right) + \frac{\sqrt{\varepsilon_\Fcal}}{1 - \gamma} + \lambda \varepsilon_r
\\
\leq &~ \frac{2 \Vmax}{1 - \gamma}\sqrt{\frac{ 2 \log|\Acal|}{T}} + \frac{\sqrt{\varepsilon_\Fcal}}{1 - \gamma} + \lambda \varepsilon_r + \frac{1}{T} \sum_{t = 1}^{T} \Bigg( \frac{2 \sqrt{C_{2,t}} (\sqrt{\varepsilon_b} + \sqrt{V_{\max}/\lambda})}{1 - \gamma}
\\
&~ + \min_{\nu:\Cscr(\nu;\mu,\Fcal,\pi_t) \leq C_{2,t}} \left|\sum_{(s,a) \in \Scal\times\Acal} (d_{\pi}\setminus\nu)(s,a) \left[f_t(s,a) - (\Tcal^{\pi_t} f_t)(s,a)\right] \right| \Bigg)
\\
= &~ \frac{2 \Vmax}{1 - \gamma}\sqrt{\frac{ 2 \log|\Acal|}{T}} + \frac{\sqrt{\varepsilon_\Fcal}}{1 - \gamma} + \lambda \varepsilon_r + \frac{1}{T} \sum_{t = 1}^{T} \Bigg( \frac{2 \sqrt{C_{2,t}} (\sqrt{\varepsilon_b} + \sqrt{V_{\max}/\lambda})}{1 - \gamma}
\\
&~ + \min_{\nu:\Cscr(\nu;\mu,\Fcal,\pi_t) \leq C_{2,t}} \left|\sum_{(s,a) \in \Scal\times\Acal} (d_{\pi}\setminus\nu)(s,a) \left[f_t(s,a) - (\Tcal^{\pi_t} f_t)(s,a)\right] \right| \Bigg),
\end{align}
where $C_{2,t}$ can be chosen arbitrarily for any $t \in [T]$. Since the complexity of $\Pi_{\textrm{SPI}}$ is at most $|\Fcal|^T$ by it definition, setting $\lambda = \sqrt[3]{\nicefrac{\Vmax}{(1 - \gamma)^2\varepsilon_r^2}}$ and plugging the definition of $\varepsilon_b$ and $\varepsilon_r$ (defined in Appendix \ref{sec:plceval}) completes the proof.
\end{proof}

\section{Linear Implementation of PSPI}
\label{sec:linear_implementation}

In this section, we provide the details of implementing PSPI with linear function approximation, that is, $\Fcal \coloneqq \{\phi(\cdot,\cdot)^\T \theta: \theta \in \RR^d\}$, where $\phi: \Scal\times\Acal \in \RR^d$ is a given feature map. %

Recall that \Eqref{eq:reg_pes_pe} is
\begin{align}
&~ f(s_0,\pi) + \lambda \Ecal(f,\pi;\Dcal)
\\
\coloneqq &~ \phi(s_0,\pi)^\T \theta
\\
&~ + \lambda \left(\E_\Dcal \left[\left(\phi(s,a)^\T \theta - r - \gamma \phi(s',\pi)^\T \theta \right)^2 \right] - \min_{\theta' \in \RR^d} \E_\Dcal \left[\left(\phi(s,a)^\T \theta' - r - \gamma \phi(s',\pi)^\T \theta \right)^2 \right]\right).
\end{align}
We first provide a closed-form solution to the inner $\min_{\theta' \in \RR^d}$. Note that this inner $\min$ is a linear regression objective, so the minimal value can be achieved with 
\begin{align}
\label{eq:theta_prime}
\theta' = \Sigma^\dagger \E_{\Dcal}[\phi(s,a) ( r + \gamma \phi(s', \pi)^\top \theta)].
\end{align}
where $\Sigma \coloneqq \E_\Dcal[\phi(s,a) \phi(s,a)^\T]$ is the sample covariance matrix and $\Sigma^\dagger$ is its pseudo-inverse. Here, we do not require the invertibility of $\Sigma$, since we only care about the $\min$ value instead of the $\argmin_{\theta'}$. The $\min$ value is therefore
\begin{align} \label{eq:inner_min}
\E_\Dcal \left[\left(\phi(s,a)^\T \Sigma^\dagger \E_{\Dcal}[\phi(s,a) ( r + \gamma \phi(s', \pi)^\top \theta)] - r - \gamma \phi(s',\pi)^\T \theta \right)^2 \right].
\end{align}
It should be clear at this point that \Eqref{eq:reg_pes_pe} is quadratic in $\theta$. The rest of the derivation provides a simplified closed-form expression. Define shorthand notation
\begin{align}
\phi \coloneqq \phi(s,a), & \qquad \psi \coloneqq \phi(s', \pi), \\
B \coloneqq \E_\Dcal \left[ \phi \psi^\T \right], & \qquad C \coloneqq \E_\Dcal \left[ \psi \psi^\T \right], ~~ b \coloneqq \E_\Dcal \left[ \phi r \right], ~~ c \coloneqq \E_\Dcal \left[ \psi r \right].
\end{align}
\newcommand{\Si}{\Sigma^\dagger}
Then, \Eqref{eq:inner_min} is
\begin{align} 
&~ \E_\Dcal \left[\left(\phi^\T \Si \E_{\Dcal}[\phi ( r + \gamma \psi^\top \theta)] - r - \gamma \psi^\T \theta \right)^2 \right] \\
= &~ \E_\Dcal \left[\left(\phi^\T \Si (b + \gamma B \theta) - r - \gamma \psi^\T \theta \right)^2 \right] \\
= &~ \E_\Dcal \left[\left(\gamma \left( \phi^\T \Si B - \psi^\T \right) \theta + \phi^\top \Si b - r\right)^2 \right] \\
= &~ \gamma^2 \E_\Dcal[ \theta^\T (B^\T \Si \phi - \psi)(\phi^\T \Si B - \psi^\T) \theta] + 2 \gamma \E_\Dcal[\theta^\T (B^\T \Si \phi - \psi) (\phi^\T \Si b - r)] + \textrm{constant},
\end{align}
where ``constant'' is any term that is independent of $\theta$ and will not affect the optimization. Dropping the constant, the above is equal to
\begin{align} 
&~ \gamma^2 \E_\Dcal[ \theta^\T (B^\T \Si \phi \phi^\T \Si B - \psi \phi^\T \Si B - B^\T \Si \phi \psi^T + \psi \psi^\T) \theta]  \\
&~~~ + 2 \gamma \E_\Dcal[\theta^\T (B^\T \Si \phi \phi^\T \Si b - \psi \phi^\T \Si b - B^\T \Si \phi r + \psi r )] \\
= &~ \gamma^2 \theta^\T (B^\T \Si \Sigma \Si B - 2 B^\T \Si B + C) \theta + 2 \gamma \theta^\T (B^\T \Si \Sigma \Si b - 2 B^\T \Si b + c) \\
= &~  \gamma^2 \theta^\T (B^\T \Si  B - 2 B^\T \Si B + C) \theta + 2 \gamma \theta^\T (c - B^\T \Si b ) \tag{Property of pseudo-inverse: $\Si \Sigma \Si = \Si$} \\
= &~ \gamma^2 \theta^\T (C - B^\T \Si B) \theta + 2 \gamma \theta^\T (c - B^\T \Si b ).
\end{align}
We now handle the first expectation in \Eqref{eq:reg_pes_pe}:
\begin{align}
&~ \E_\Dcal \left[\left(\phi(s,a)^\T \theta - r - \gamma \phi(s',\pi)^\T \theta \right)^2 \right] \\
= &~ \E_\Dcal \left[\left((\phi^\T - \gamma \psi^\T) \theta - r \right)^2 \right] \\
= &~ \E_\Dcal \left[\theta^\T (\phi - \gamma \psi)(\phi^\T - \gamma \psi^\T) \theta\right]  - 2 \E_\Dcal[ r \cdot (\phi - \gamma \psi)^\T \theta ]   + \textrm{constant} \\
= &~\theta^\T (\Sigma - \gamma B^\T - \gamma B + \gamma^2 C) \theta  - 2 (b - \gamma c) \theta   + \textrm{constant}.
\end{align}
We now combine the two terms and consider the quadratic term $\theta^\T (\cdot) \theta$ and the linear term $ \theta^\T(\cdot)$ separately. The matrix in the quadratic term is
\begin{align}
&~\Sigma - \gamma B^\T - \gamma B + \gamma^2 C - \gamma^2(C - B^\T \Si B) \\
= &~ \Sigma - \gamma B^\T - \gamma B + \gamma^2 B^\T \Si B \\
= &~  (I - \gamma\Si B)^\T \Sigma(I - \gamma \Si B).
\end{align}
To verify the last step, we can expand the last expression and obtain $\Sigma - \gamma\Sigma \Si B - \gamma (\Sigma \Si B)^\T + \gamma^2 B^\T \Si B$, and the identity holds due to $\Sigma \Si B = B$; we defer the proof of this fact to the end of this section. 

The vector in the linear term is
\begin{align}
-2(b - \gamma c) - 2 \gamma (c - B^\T \Si b)
= - 2 (I - \gamma B^\T \Si)b.
\end{align}
Putting all together, \Eqref{eq:reg_pes_pe} divided by $\lambda$ is equal to (up to a constant independent of $\theta$):
\begin{align} \label{eq:linear_loss}
&~\theta^\T (I - \gamma\Si B)^\T \Sigma(I - \gamma \Si B) \theta - \theta^\T \left(2 (I - \Si \gamma B)^\T b - \phi(s_0, \pi)/\lambda\right).
\end{align}
Note that this objective is quadratic in $\theta$, and the Hessian is always positive semi-definite. 

\paragraph{Closed-form solution under invertibility and connection to LSTDQ} We show that the above objective is intimately connected to LSTDQ \citep{lagoudakis2003least}. In particular, assuming $\Sigma$ and $\Sigma - \gamma B$ are both invertible, \Eqref{eq:linear_loss} becomes
\begin{align}
\theta^\T (I - \gamma\Sigma^{-1} B)^\T \Sigma(I - \gamma \Sigma^{-1} B) \theta -  \theta^\T \left( 2 (I - \gamma \Sigma^{-1}B)^\T b - \phi(s_0, \pi)/\lambda \right).
\end{align}
Note that the quadratic term is now positive definite, and we are minimizing the objective, so the minimizer can be found simply by setting the gradient to $0$, i.e.,
\begin{align}
&~2(I - \gamma\Sigma^{-1} B)^\T\Sigma(I - \gamma \Sigma^{-1} B) \theta = 2 (I - \gamma \Sigma^{-1}B)^\T b - \phi(s_0, \pi)/\lambda.
\end{align}
Define $A \coloneqq I - \gamma \Sigma^{-1} B$, the closed-form solution is
\begin{align}
\theta = &~  A^{-1} \Sigma^{-1} (A^{-1})^\T (A^\T b - \phi(s_0, \pi)/2\lambda) \\
= &~ A^{-1} \Sigma^{-1} b- A^{-1} \Sigma^{-1} (A^{-1})^\T  \phi(s_0, \pi) / 2\lambda.
\end{align}
Note that when we drop the pessimistic term $f(s_0, \pi)$ (i.e., setting $\lambda \to \infty$), the solution becomes $A^{-1} \Sigma^{-1} b$, which is exactly LSTDQ \cite{lagoudakis2003least}. \citet[Proposition 2]{antos2008learning} shows a similar result, but their proof is restricted to the invertible case and directly verifies that \Eqref{eq:reg_pes_pe} achieves its minimal possible value $0$ when the LSTDQ solution is plugged in. In contrast, our result in \Eqref{eq:linear_loss} is substantially more general as it does not rely on the invertibility assumptions.

\paragraph{Proof of $\Sigma \Si B = B$} We rewrite $\Sigma = \E_{\Dcal}[\phi\phi^\T] = \Phi^\T \Dcal \Phi$, where $\Dcal$ is a diagonal matrix representing the empirical measure of $\Dcal$ over $\Scal\times\Acal$, and $\Phi \in \RR^{|\Scal\times\Acal| \times d}$ is the matrix representation of the entire feature map. Similarly we may write $B = \E_{\Dcal}[\phi\psi^\T] = \Phi^\T \Dcal \Psi$ for some suitable matrix $\Psi$.\footnote{The $(s,a)$-th row of $\Psi$ is $\E_\Dcal[\phi(s', \pi)^\T | s, a]$ for $(s,a)$ in the data, and does not matter otherwise.} Define $X = \Dcal^{1/2} \Phi$ and $Y = \Dcal^{1/2} \Psi$, and $\Sigma \Si B = B$ becomes $(X^\T X) (X^\T X)^\dagger (X^\T Y) = X^\T Y$, which is what we need to show. 

To show this, let $X = U Z V^\T$ be the SVD of $X$, where $Z \in \RR^{r\times r}$ is an invertible diagonal matrix with $r$ being the rank of $X$. Note that $U^\T U = V^\T V = I_r$, i.e., the $r\times r$ identity matrix. Then,
\begin{align}
&~ (X^\T X) (X^\T X)^\dagger (X^\T Y) \\
= &~ V Z^2 V^\T V Z^{-2} V^\T V Z U^\T Y \\
= &~ V V^\T V Z U^\T Y\\
= &~ V Z U^\T Y = X^\T Y.
\end{align}
This completes the proof.

\end{document}